\documentclass{article} 
\usepackage{iclr2024_conference,times}

\usepackage{amsmath,amsfonts,bm}

\def\eqref#1{equation~\ref{#1}}

\def\1{\bm{1}}

\DeclareMathAlphabet{\mathsfit}{\encodingdefault}{\sfdefault}{m}{sl}
\SetMathAlphabet{\mathsfit}{bold}{\encodingdefault}{\sfdefault}{bx}{n}

\newcommand{\R}{\mathbb{R}}

\usepackage{hyperref}
\usepackage{url}
\usepackage{url}
\usepackage{microtype}
\usepackage{graphicx}
\usepackage{multicol}
\usepackage{wrapfig}
\usepackage{multirow}
\usepackage{booktabs}
\usepackage[utf8]{inputenc}
\usepackage[T1]{fontenc}   
\usepackage{amsfonts}       
\usepackage{nicefrac}      
\usepackage{microtype,csquotes}      
\usepackage{xcolor}         
\usepackage{graphicx, amsfonts, amsmath,amsthm, amssymb} 
\usepackage{xcolor}
\usepackage{algorithm}
\usepackage[noend]{algpseudocode}
\usepackage{mathrsfs, comment}
\usepackage{framed}
\usepackage{varwidth}
\usepackage{nccmath, mathtools}

\usepackage{bbm}
\usepackage{booktabs, bigstrut}
\usepackage{enumitem}
\usepackage{paralist}
\usepackage{xspace}

\newtheorem{defn}{Definition}
 
\newtheorem{theorem}{Theorem}[section]

\newtheorem{lemma}[theorem]{Lemma}
\newtheorem{fact}[theorem]{Fact}
\newenvironment{proofs}{%
  \proof}{\endproof}
\usepackage{subcaption}
\def \P{\mathcal{P}}
\def \PC{\mathbb{P}}
\def \m{f}

\def \D{\mathcal{D}}
\def \S{\mathcal{S}}
\def \C{\mathscr{C}}
\def \x{x^*}
\def \N{\mathcal{N}}

\def \F{\mathcal{F}}

\def \name{\textit{FairProof}\xspace}
\def \calA{\mathcal{A}}
\def \W{\mathbf{W}}
\def \R{\mathbf{R}}
\def \RQ{\mathbf{R'}}
\def \A{\mathbf{A}}

\def \CW{\textsf{com}_{\mathbf{W}}}

\def \calS{\mathcal{S}}

\def \G{\scalebox{0.9}{\textsf{GeoCert}}}
\def \dproj{d_{proj}}
\def \dlp{d_{\ell_2}}
\def \VNeighbor{\scalebox{0.9}{\textsf{VerifyNeighbor}}}
\def \VPoly{\scalebox{0.9}{\textsf{VerifyPolytope}}}
\def \VBoundary{\scalebox{0.9}{\textsf{VerifyBoundary}}}
\def \VPQ{\scalebox{0.9}{\textsf{VerifyOrder}}}
\def \VDist{\scalebox{0.9}{\textsf{VerifyDistance}}}
\def \VMin{\scalebox{0.9}{\textsf{VerifyMin}}}
\def \VInf{\scalebox{0.9}{\textsf{VerifyInference}}}

\newcommand{\squishlist}{
	\begin{list}{$\bullet$}
		{
			\setlength{\itemsep}{0pt}
			\setlength{\parsep}{1pt}
			\setlength{\topsep}{1pt}
			\setlength{\partopsep}{0pt}
			\setlength{\leftmargin}{2.5em}
			\setlength{\labelwidth}{1em}
			\setlength{\labelsep}{0.5em} } }
	
\newcommand{\squishend}{
\end{list}  }

\newcommand{\arc}[1]{\textcolor{blue}{ ARC: #1}}

\title{\name: Confidential and Certifiable Fairness for Neural Networks}

\author{Chhavi Yadav$^1$, Amrita Roy Chowdhury$^1$, Dan Boneh$^2$, Kamalika Chaudhuri$^1$\thanks{Corresponding author : \texttt{cyadav@ucsd.edu}}\\
$^1$University of California, San Diego $^2$ Stanford University\\
}

\iclrfinalcopy 
\begin{document}

\maketitle

\begin{abstract}
Machine learning models are increasingly used in societal applications, yet legal and privacy concerns demand that they very often be kept confidential. Consequently, there is a growing distrust about the fairness properties of these models in the minds of consumers, who are often at the receiving end of model predictions. To this end, we propose \name -- a system that uses Zero-Knowledge Proofs (a cryptographic primitive) to publicly verify the fairness of a model, while maintaining confidentiality. We also propose a fairness certification algorithm for fully-connected neural networks which is befitting to ZKPs and is used in this system. We implement \name in Gnark and demonstrate empirically that our system is practically feasible. Code is available at \url{https://github.com/infinite-pursuits/FairProof}.
\end{abstract}

\section{Introduction}

Recent usage of ML models in high-stakes societal applications ~\cite{creditprediction,crimeprediction,adprediction} has raised serious concerns about their fairness ~\citep{unfair1, AppleCard, resume,jobunfair}. As a result, there is growing distrust in the minds of a consumer at the receiving end of ML-based decisions ~\cite{dwork2022distrust}. In order to increase consumer trust, there is a need for developing technology that enables public verification of the fairness properties of these models.

A major barrier to such verification is that legal and privacy concerns demand that models be kept confidential by organizations. The resulting lack of verifiability can lead to potential misbehavior, such as  model swapping, wherein a malicious entity uses different models for different customers leading to unfair behavior. Therefore what is needed is a solution which allows for public verification of the fairness of a model and ensures that the same model is used for every prediction (model uniformity) while maintaining model confidentiality. The canonical approach to evaluating fairness is a statistics-based third-party audit  ~\cite{yadav2022learningtheoretic, yan2022active, pentyala2022privfair, soares2023keeping}. This approach however is replete with problems arising from the usage of a reference dataset, the need for a trusted third-party, leaking details about the confidential model~\cite{casper2024black, hamman2023can} and lack of guarantees of model uniformity~\cite{fukuchi2019faking, confidant}.

We address the aforementioned challenges by proposing a system called \textit{FairProof} involving two parts: 
1) a fairness certification algorithm which outputs a certificate of fairness , and 
2) a cryptographic protocol using commitments and Zero-Knowledge Proofs (ZKPs) that guarantees model uniformity and gives a proof that the certificate is correct.

\begin{figure}[tb]
    \centering \includegraphics[width=\linewidth]{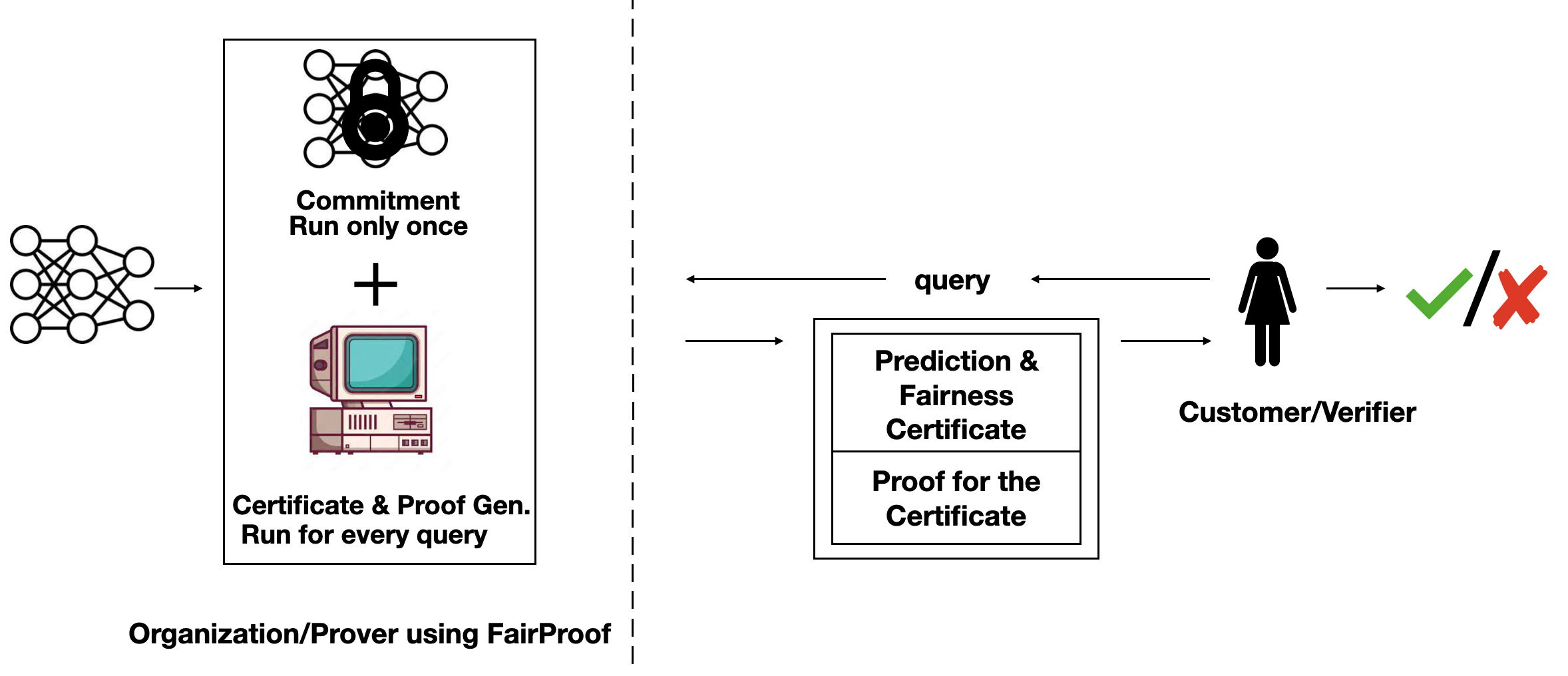}  
        \caption{Pictorial Representation of \name}
        \vspace{-17pt}
        \label{fig:polytope}
\end{figure}

Given an input query, the fairness certification algorithm outputs how fair the model is at that point \textit{according to a fairness metric}. The metric we use is local Individual Fairness (IF)~\cite{Dwork12,john2020verifying,Benussi2022IndividualFG,IF}, which is desirable for two reasons. First, it evaluates fairness of the model at a specific data point (rather than for the entire input space) -- this allows us to give a \textit{personalized} certificate to every customer, as would be required by customer-facing organizations.  Second, it works on the model post-training, making it completely \textit{agnostic} to the training pipeline.

How do we design a certification algorithm for the chosen metric? We observe that certifying local IF can be reduced to an instantiation of certifying robustness.\footnote{Certifiable Robustness quantifies a model's resistance to adversarial attacks by measuring the extent to which a data point can be perturbed without altering the model prediction.} We then leverage techniques from the robustness literature to design our algorithm. One of our key contributions is to design the algorithm so that it is ZKP-friendly. In particular, the computational overhead for ZKPs depends on the complexity of the statement being proved. To this end, we design a fairness certificate which results in relatively low complexity statements.

Once the fairness certificate has been computed, we want to enable the consumer to verify that the certificate was indeed computed correctly, but without revealing the model weights. To do this, we rely on Succinct Zero Knowledge Proofs ~\cite{GMR,GMW}. This cryptographic primitive enables a prover (eg. bank) to prove statements (eg. fairness certificate) about its private data (eg. model weights) without revealing the private data itself. It provides a proof of correctness as an output. Then a verifier (eg. customer) verifies this proof without access to the private data. In our case, if the proof passes verification, it implies that the fairness certificate was computed correctly with respect to the hidden model.

We design and implement a specialized ZKP  protocol to efficiently prove and verify the aforementioned fairness certification algorithm. Doing this naively would be very computationally expensive. We tackle this challenge with three insights.  First, we show that verification of the entire certification algorithm can be reduced to a few strategically chosen sub-functionalities, each of which can be proved and verified efficiently. Second, we provide a \textit{lower} bound on the certificate, i.e., a conservative estimate of the model's fairness, for performance optimization. Third, we observe that certain computations can be done in an offline phase thereby reducing the online computational overhead. 

Our solution ensures model uniformity through standard cryptographic commitments. A cryptographic commitment to the model weights binds the organization to those weights publicly while maintaining confidentiality of the weights. This has been widely studied in the ML security literature~\citep{gupta2023sigma, boemer2020mp2ml, juvekar2018gazelle, liu2017oblivious, srinivasan2019delphi, mohassel2017secureml, mohassel2018aby3}.


\textbf{Experiments.} In this work we focus on fully-connected neural networks with ReLU activations as the models. We implement and evaluate \name~on three standard fairness benchmark datasets to demonstrate its practical feasibility. For instance, for the \textit{German}~\cite{German} dataset, we observe that \name~takes around 1.17 minutes on an average to generate a verifiable fairness certificate per data point without parallelism or multi-threading on an Intel-i9 CPU chip. The communication cost is also low -- the size of the verifiable certificate is only 43.5KB.
\section{Preliminaries \& Setting}\label{sec:background}
\textbf{Fairness.}\label{sec:background:fairness}
Existing literature has put forth a wide variety of fairness definitions~\cite{survey,barocas-hardt-narayanan}. In this paper, we focus on the notion of \textit{local individual fairness}~\cite{john2020verifying,Dwork12, Benussi2022IndividualFG} defined below, as it best aligns with our application (see Sec. \ref{sec:setting} for more details).
\begin{defn}[Local Individual Fairness] A machine learning model $\m: \mathbb{R}^n \mapsto \mathcal{Y}$ is defined to be $\epsilon$-individually fair w.r.t to a data point $\x\sim \D$ under some distance metric $d:\mathbb{R}^n \times \mathbb{R}^n \mapsto \mathbb{R}$  if
\begin{gather}
  \forall x: \ d(x,\x)\leq \epsilon \implies \m(\x)=\m(x)\label{eq:IF:1}
\end{gather}\label{def:IF}
\end{defn}
\vspace{-22pt}
We say a model $\m$ is exactly $\epsilon^*$-individually fair w.r.t $\x$ if $\epsilon^*$ is the largest value that satisfies Eq. \ref{eq:IF:1}.  In particular, $\epsilon^*$ is known as the local individual fairness parameter. \textbf{For brevity we will be using $\epsilon$ to mean $\epsilon^*$ and fairness/individual fairness to refer to the notion of local individual fairness, unless stated otherwise, throughout the rest of the paper.}

Individual fairness formalizes the
notion that similar individuals should be treated similarly; more precisely, get the same classification. The similarity is defined according to a task dependent distance metric $d(\cdot)$ that can be provided
by a domain expert. Examples of such a metric could be weighted $\ell_p$ norm where the weights of the sensitive features (race, gender) are set to 0~\cite{Benussi2022IndividualFG}. 

\begin{figure}[tb]
    \centering \includegraphics[width=0.3\linewidth]{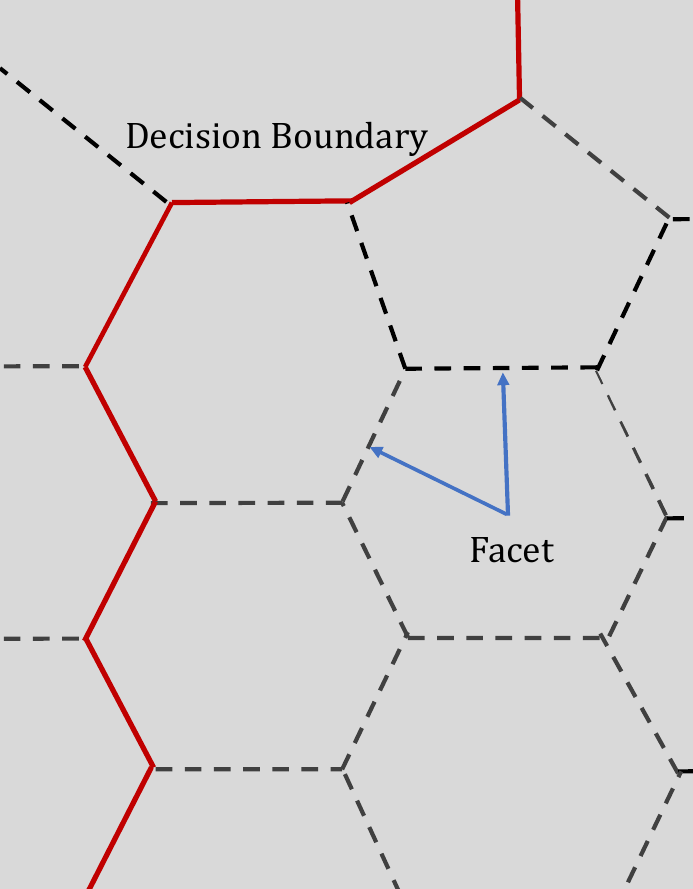}  
        \caption{A neural network with ReLU activations partitions the input space into polytopes.}
        \label{fig:polytope}
\end{figure}

\textbf{Neural Networks.}\label{sec:background:PLNN}
We focus on the classification task and 
consider neural network (NN) classifiers $\m: \mathcal{X} \mapsto \mathcal{Y}$, where $f$ is a fully-connected neural network with ReLU activations, $\mathcal{X} = \mathbb{R}^n$ is the input space and $\mathcal{Y}$ is a discrete label set. This NN classifier (pre-softmax) can also be viewed a collection of piecewise linear functions over a union of convex polytopes~\cite{xu2021traversing,hanin2019deep,robinson2019dissecting,croce2019provable,serra2018bounding}. Here each linear function corresponds to one polytope and each polytope corresponds to one activation pattern of the nodes in the NN. A polytope $\P$ is represented by a set of linear inequalities, $\P=\{x|\mathbf{A}x\leq \mathbf{b}\}$ ; then the collection of all such polytopes forms a partition of the input domain, $\mathcal{X} = \bigcup\P$ (Fig. \ref{fig:polytope}).

A facet is an $(n-1)$-\textit{face} of the polytope corresponding to the set $\{x|x \in \P \cap \mathbf{A}_ix = \mathbf{b}_i\}$ where $\mathbf{A}_i$ and $\mathbf{b}_i$ are the values of $\mathbf{A}$ and $\mathbf{b}$ at the $i^{th}$ dimension. Two polytopes that share a facet are known as neighboring polytopes. The decision region of $\m$ at a data point $\x$ is defined as the set of points for which the classifier returns the same
label as it does for $\x$, essentially  the set $ \{x | \m(x) = \m(\x)\}$. This decision region can also be expressed as a union of convex polytopes~\cite{Geocert}. A facet that coincides with the decision boundary of $\m$ is known as a boundary facet. See Fig.~\ref{fig:polytope}  and App. \ref{app:background} for more details.


\textbf{Cryptographic Primitives.}\label{sec:background:ZKP} We use two cryptographic primitives, namely commitment schemes and zero knowledge proof,  for verifying the individual fairness certification.

\noindent A \textit{Commitment Scheme} commits to a private input~$w$ without revealing anything about $w$; its output is a commitment string $\textsf{com}_{w}$. A commitment scheme has two  properties: 
\begin{compactenum}
    \item \textit{Hiding}: the commitment string $\textsf{com}_{w}$ reveals nothing about the committed value $w$. 
    \item \textit{Binding}: it is not possible to come up with another input $w'$ with the same commitment string as $w$, thus binding $w$ to $\textsf{com}_{w}$ (simplified). 
\end{compactenum}


\textit{Zero Knowledge Proofs}~\cite{GMR} describe a protocol between two parties -- a prover and a verifier, who both have access to a circuit $P$. A ZKP protocol enables the prover to convince the verifier that it \textit{possesses} an input $w$ such that $P(w) = 1$,
without revealing any additional information about $w$ to the verifier. 
A simple example is when $P_\varphi(w) = 1$ iff $\varphi$ is a SAT formula and $\varphi(w) = 1$;
a ZKP protocol enables the prover to convince a verifier that there is a $w$ for which $\varphi(w) = 1$,
while revealing nothing else about $w$. A ZKP protocol has the following properties:


\begin{compactenum}
\item\textit{Completeness}. For any input $w$ such that $P(w) = 1$, an honest prover who follows the protocol correctly can convince an honest verifier that $P(w) = 1$. 

\item\textit{Soundness.} Given an input $w$ that $P(w) \neq 1$, a malicious prover who deviates arbitrarily from the protocol cannot falsely convince an honest verifier that $P(w) = 1$, with more than negligible probability. 

\item\textit{Zero knowledge.} If the prover and verifier execute the protocol to prove that $P(w) = 1$, even a malicious
verifier, who deviates arbitrarily from the protocol, can learn no additional information about $w$ other than $P(w) = 1$.
\end{compactenum}

Theory suggests that it is possible to employ ZKPs to verify any predicate $P$ in the class NP~\cite{GMW}. Moreover, the resulting proofs are non-interactive and succinct. However, in practice, generating a proof for even moderately complex predicates often incurs significant computational costs. To this end, our main contribution lies in introducing a ZKP-friendly certification algorithm, to facilitate efficient fairness certificate generation.

\textbf{Problem Setting.}\label{sec:setting} Formally, a model owner holds a confidential classification model $\m$ that cannot be publicly released. A user supplies an input query $\x$ to the model owner, who provides the user with a prediction label $y=\m(\x)$ along with a fairness certificate $\C$ w.r.t to $\x$. This certificate can be verified by the user and the user is also guaranteed that the model owner uses the same model for everyone.

The above setting needs three tools. First, the model owner requires an algorithm for generating the fairness certificate with white-box access to the model weights. This algorithm is discussed in Sec. \ref{sec:certification}. Second, a mechanism is needed that enables the user to \textit{verify} the received certificate (public verification) \textit{without violating model confidentiality}. This mechanism is discussed in Sec. \ref{sec:verification}. Third, a mechanism is needed to guarantee that the same model is used for everyone (model uniformity), also without violating model confidentiality. For ensuring uniformity, the model owner should commit the model in the initially itself, before it is deployed for users. This has been widely studied and implemented by prior work as discussed in the introduction and an actual implementation of commitments is out of scope of this work.

\section{How to Certify Individual Fairness?}\label{sec:certification}

In this section we present an algorithm to compute a local individual fairness certificate. This certificate is computed by the model owner with white-box access to the model weights and is specific to each user query, thereby leading to a \textit{personalized} certificate. The certificate guarantees to the user that the model has certain fairness properties at their specific query. 


\textit{Preliminaries.} Starting with some notation, let $\mathcal{S}$ be the set of $k$ sensitive features, $\mathcal{S}:=\{S_1, \cdots, S_k\}$ where $S_i$ denotes the $i^{th}$ sensitive feature. We assume that each sensitive feature $S_i$ has a discrete and finite domain, denoted by $domain(S_i)$, which is in line with typical sensitive features in practice, such as race (eg. black/white/asian), presence of a medical condition (yes/no). Let $domain(\mathcal{S})$ represent the set of all possible combinations of the values of sensitive features, $domain(\mathcal{S}):= domain(S_1)\times \cdots \times domain(S_k)$. Without loss of generality, any data point $x \in \mathbb{R}^n$ is also represented as $x = x_{\setminus \mathcal{S}} \cup x_{\mathcal{S}}$, where $x_{\setminus \mathcal{S}}$ and $x_{\mathcal{S}}$ are the non-sensitive and sensitive features of $x$.


For the distance metric in individual fairness (Eq. \ref{eq:IF:1}), we consider a weighted $\ell_2$-norm where the non-sensitive features have weight $1$ while the sensitive features have weight $0$. This distance metric is equivalent to the $\ell_2$-norm sans the sensitive features. Thus, based on Def. \ref{def:IF}, $\m$ is $\epsilon$-individually fair w.r.t  $\x$ iff,
\begin{gather}
\forall x: \thinspace  ||x_{\setminus\mathcal{S}} -\x_{\setminus\mathcal{S}}||_2\leq \epsilon \implies \m(\x)=\m(x)\label{eq:IF} 
\end{gather}

With this notation in place, observe that our fairness certificate $\C$ is essentially the value of the parameter $\epsilon$. Intuitively it means that the model's classification is independent of the sensitive features as long as the non-sensitive features lie within an $\ell_2$ ball of radius $\epsilon$ centered at $\x_{\setminus\mathcal{S}}$. Eq.\ref{eq:IF} can also be equivalently viewed as follows: set the sensitive features of $\x$ and $x$ to a particular value $s \in domain(\mathcal{S})$ (so that they cancel out in the norm), then find the corresponding certificate $\epsilon_s$ and repeat this procedure for all values in $domain(\mathcal{S})$; the final certificate $\epsilon$ is the minimum of all $\epsilon_s$.

Next we propose an algorithm to compute this fairness certificate. Our algorithm is based on three key ideas, as we describe below.

    

\begin{figure}
\centering
  \begin{subfigure}{0.45\linewidth}
    \centering \includegraphics[width=0.75\linewidth]{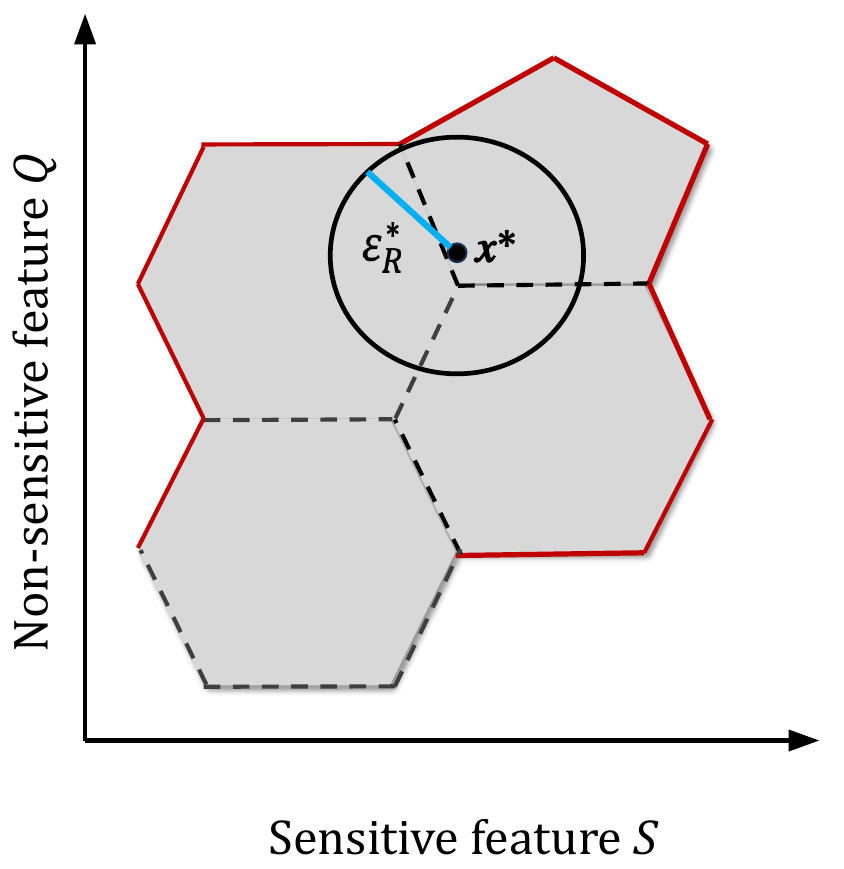}  
        \caption{Certifiable robustness}
        \label{fig:robustness}\end{subfigure}
        \begin{subfigure}{0.45\linewidth}
        \includegraphics[width=0.75\linewidth]{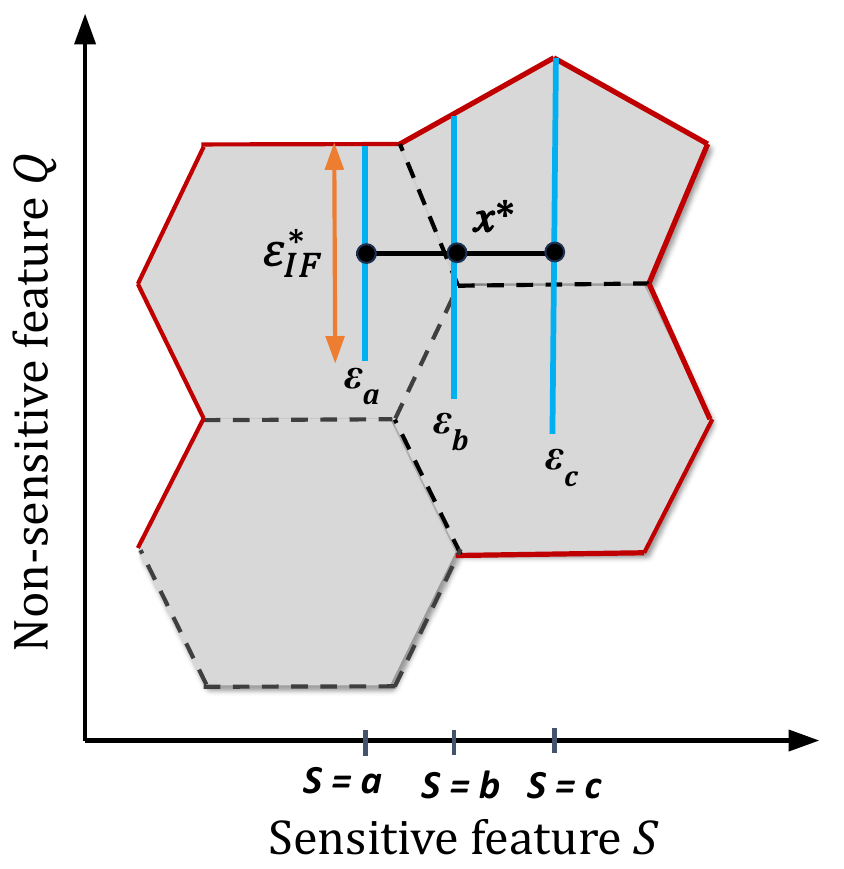}  
        \caption{Certifiable fairness}
        \label{fig:fairness}
        \end{subfigure} 
\caption{Connection between robustness \& fairness for $n=2$ and one sensitive feature $S$ with values $\{a,b,c\}$. Final fairness certificate is the minimum of $\{\epsilon_a, \epsilon_b, \epsilon_c\}$. Red color denotes decision boundary. }\label{fig:connection}
\end{figure}

\noindent\textbf{Idea 1: Reduction from fairness to robustness.} Our first key observation is that in our setting, certifiable fairness can be reduced to an instantiation of certifiable robustness, which enables us to re-use ideas from existing robustness literature for our purpose. In particular, the reduction is as follows. 
A model $\m$ is defined to be $\epsilon$-pointwise $\ell_2$ robust (henceforth robustness) for a data point $\x$, if
\begin{equation}
\forall x: \thinspace  ||x-\x||_2\leq \epsilon \implies \m(\x)=\m(x)
\end{equation}

Comparing this definition to Eq.\ref{eq:IF} and its alternate view, we observe that once the sensitive features have been fixed to a value $s \in domain(\mathcal{S})$, computing the corresponding fairness certificate $\epsilon_s$ is \textit{equivalent} to solving the robustness problem in $(n-k)$ dimensions where the $k$ dimensions corresponding to the sensitive features $\S$ are excluded. Let us assume there exists an algorithm which returns the pointwise $\ell_2$ robustness value for an input. Then the final fairness certificate $\epsilon$ computation requires \scalebox{0.95}{$|domain(\mathcal{S})|$} calls to this algorithm, one for each possible value of the sensitive features in $\S$. Fig. \ref{fig:connection} illustrates this idea pictorially for NNs.  

For ReLU-activated neural networks represented using $n$-dimensional polytopes, setting the values of sensitive features implies bringing down the polytopes to $(n-k)$ dimensions. Geometrically, this can be thought of as slicing the $n$-dimensional polytopes with hyperplanes of the form $x_i = s_i$ where $x_i$ is the $i^{th}$ coordinate, set to the value $s_i$.



\begin{algorithm}[tbh]
  
\begin{algorithmic}[1]
 \caption{ Individual Fairness Certification}
   \label{alg:IF}
\Statex\textbf{Inputs} $\x \in \mathbb{R}^n$, $f$ : ReLU-activated Neural Network
\Statex \textbf{Output} $\epsilon_{LB}$ : Our Fairness Certificate for $\x$
\State Construct the set of all polytopes \scalebox{0.9}{$\PC=\bigcup\P$} for $\m$ where each polytope is expressed as \scalebox{0.9}{$\P=\{x|\mathbf{A}x\leq \mathbf{b}\}$}
\State $E:=[\hspace{0.1cm}]$
\State \textbf{for} $s \in domain(S_1)\times\cdots\times domain(S_k)$
\State \hspace{0.3cm} $\PC' := $ \scalebox{0.9}{\textsf{ReducePolyDim}}($\PC$, $s$) (Alg. \ref{alg:app:n_kpoly} in Appendix)
\State \hspace{0.3cm} $\epsilon_s := \G(\x,\PC', \dproj)$

\State \hspace{0.3cm} $E.append(\epsilon_s)$
\State \textbf{end for}
\State $\epsilon_{LB} := \min E$
\State \textbf{Return} $\epsilon_{LB}$
\end{algorithmic}
\end{algorithm}

\textbf{Idea 2. Using an efficient certified robustness algorithm.} For ReLU-activated neural networks (see Sec.\ref{sec:background:PLNN}), the naive algorithm for certifying robustness is infeasible; it entails computing the distance between $\x$ and \textit{all}  boundary facets (facets coinciding with the decision boundary of the model) induced by the model, which is exponential in the number of hidden neurons. Instead, we rely on an efficient iterative algorithm \textit{GeoCert} (Alg. \ref{alg:geo} in App. \ref{app:IF}), proposed by ~\cite{Geocert}. This algorithm starts from the polytope containing the data point $\x$ and \textit{iteratively} searches for the boundary facet with the minimum distance from $\x$. A priority queue of facets is maintained, sorted according to their distance from $\x$. At each iteration, the facet with the minimum distance is popped and its neighbors (polytopes adjacent to this facet) are 
examined. If the neighboring polytope is previously unexplored, the distance to all of its facets is computed and inserted them into the
priority queue; otherwise the next facet is popped. The algorithm terminates as soon as a boundary facet is popped. Fig. \ref{fig:geocert} presents a pictorial overview of \textit{GeoCert}. Additional details are in App. \ref{app:IF}.

\begin{figure}[tb]
    \centering \includegraphics[width=00.7\linewidth]{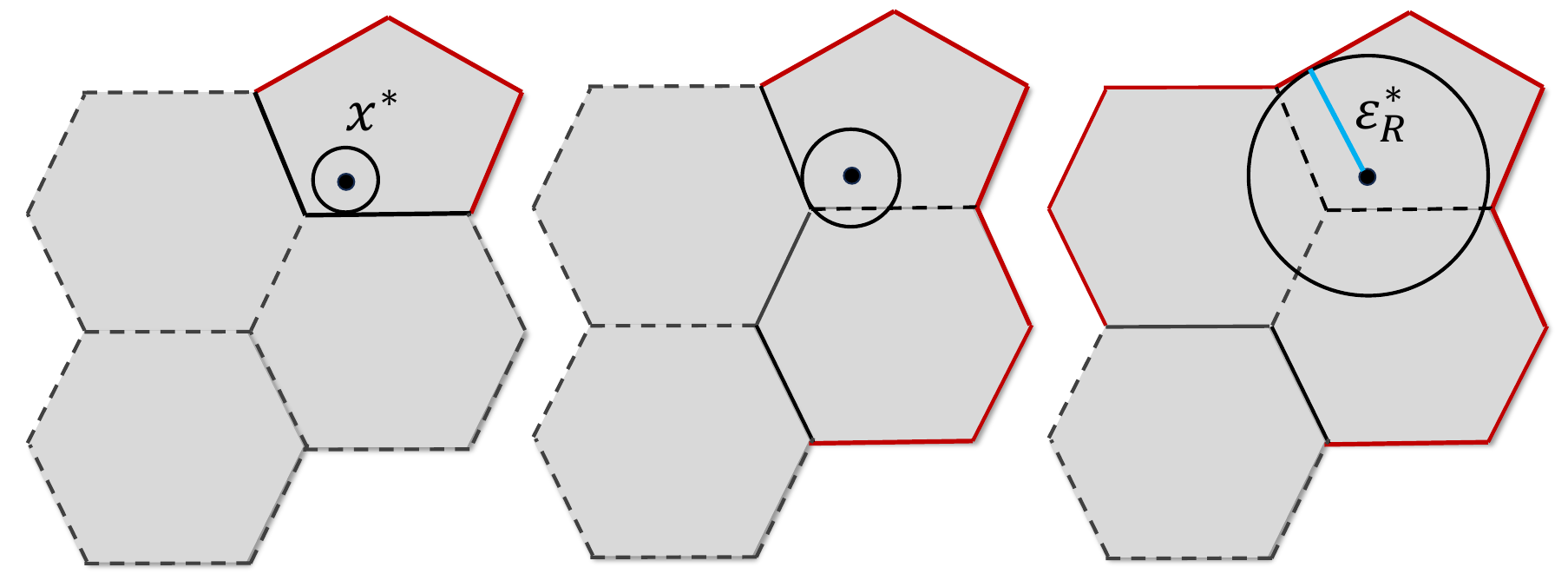}  
        \caption{\textit{GeoCert}'s behavior on point $\x$. Colored facets are in the priority queue; red and solid black lines
denote boundary and non-boundary facets respectively. Algorithm stops when the minimum distance facet is a boundary facet (rightmost).}
        \label{fig:geocert}
\end{figure}


\textbf{Idea 3: Generate a lower bound $\epsilon_{LB}$ for efficient ZKP.} \textit{GeoCert} provides exact fairness certificates $\epsilon^*$, by using a constrained quadratic program solver to get the actual distance between the input point and a facet. However, verifying this solver using ZKPs would be a highly computationally intensive task. Instead we propose to report a lower bound on the  certificate, $\epsilon_{LB}<\epsilon^*$, which considerably improves performance. A lower bound means that the reported certificate $\epsilon_{LB}$ is a conservative estimate -- the true measure of the model's fairness could only be higher. Instead of the exact distance, we compute the \textit{projection} distance between the input point and the \textit{hyperplane} containing the facet (facet is a subset of the hyperplane), which gives a lower bound on the exact distance between $\x$ and the facet. The projection distance computation involves simple arithmetic operations which are relatively computationally feasible for ZKPs (see Sec. \ref{sec:verification} for more details). Fig. \ref{fig:projection} shows the intuition pictorially.


\begin{figure}[tb]
    \centering \includegraphics[width=0.5\linewidth]{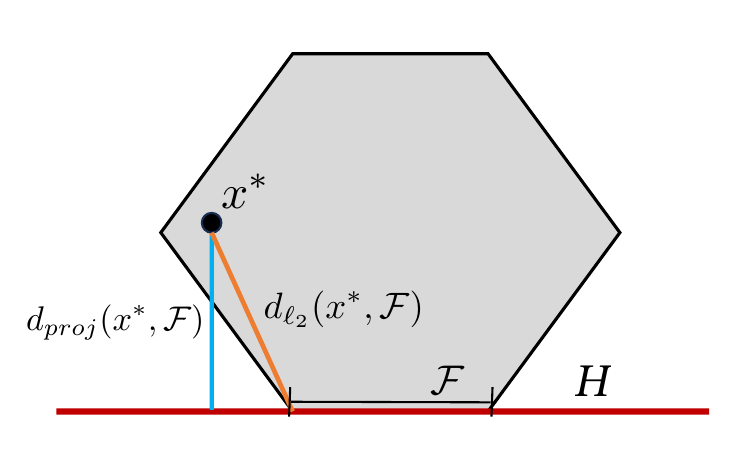}  
        \caption{Projection of $\x$ onto the hyperplane $H$ containing facet $\F$ gives a lower bound on the $\ell_2$ distance between $\x$ and $\F$, i.e., $\dproj(\x,\F)\leq \dlp(\x,\F)$.}
        \label{fig:projection}
\end{figure}

\begin{theorem}
Given a data point $\x$ and a neural network $\m$, Alg. \ref{alg:IF} provides a lower bound $\epsilon_{LB}$ of the correct individual fairness parameter of $\x$.
\end{theorem}
\vspace{-10pt}
Proof for this theorem is given in App. \ref{app:sec:correctness_fairproof_proj}, Thm. \ref{thm:projection_distance}.

Our resulting fairness certification algorithm is described in Alg.\ref{alg:IF} and detailed in App. \ref{app:IF}.


\section{\name: Verification of the Individual Fairness Certificate} \label{sec:verification}
Without careful design choices ZKPs can impose significant computational overhead. To this end, we design an efficient verification protocol named \name~ by combining insights from cryptography and  ML. Specifically, \name~is based on three key ideas described below.
 
 \noindent 
 \textbf{Idea 1: Strategic verification of sub-functionalities.} A naive verification mechanism replicates all the computations outlined in Alg.\ref{alg:IF}. However, this would involve computing \textit{all} the polytopes during \textit{every} proof generation -- this is computationally expensive since the number of polytopes is exponential in the number of hidden neurons in the model. In contrast, we show that the verification can be streamlined by focusing on five strategically chosen sub-functionalities, each of which can be checked using certain properties of polytopes and neural networks.  Consequently, we only verify the polytopes \textit{traversed} by the certification mechanism. 
 
 \noindent
 \textbf{Idea 2: Representative points.} Certain numeric properties of a polytope can be efficiently proven if one has access to a representative point in the interior of the polytope.  We leverage this insight in \name~to efficiently verify our chosen sub-functionalities, discussed in the following sections.
 
 \noindent\textbf{Idea 3: Offline computation.} We show that certain computations can be performed offline which further reduces the time needed in the online phase.

 Next, we detail our verification mechanism \name. Recall that in our setting model owner is the prover and user is the verifier. The verification consists of two phases:
 
 \textbf{Phase 1: Pre-processing.} All the operations in this phase are executed only once and before the model is deployed to the users. The following two actions need to be taken by the model owner in this phase.
\begin{compactenum}
    \item Commit to the weights $\W$ of the model $\m$, resulting in the commitment $\textsf{com}_\W$ (we assume that the architecture of $\m$ is known, i.e., $\m$ is a fully connected neural lnetwork with ReLU activations).
    \item Compute a representative point $z_{\P}$ for each polytope $\P$. Additionally, it computes a representative point $z_{\F}$ for every facet $\F$\footnote{A facet is also essentially a polytope, albeit in the $(n-1)$-dimensional space.}\footnote{Although the number of polytopes and facets are exponential in the number of the neurons in the model, this is a one-time computation performed completely offline and can be parallelized. See Sec. \ref{sec:evaluation} for empirical overhead of this pre-processing step on models for standard datasets.}.  
\end{compactenum}

\noindent\textbf{Phase 2: Online verification.} The online verification phase is executed \textit{every} time a user submits a query $\x$ to the model owner for prediction. Verifying the computation of Algorithm \ref{alg:IF} essentially amounts to verifying \textit{GeoCert} with some modifications and consists of five steps. The model owner generates proofs for these five functionalities and the user validates them.

 \textit{1. Verifying initial polytope (Alg. \ref{alg:verify:polytope}).} Recall that \textit {GeoCert} starts from the polytope containing data point $\x$. Hence, the verifier needs to check that the initial polytope
  $(1)$ indeed contains the data point $\x$, and
    $(2)$ is one of the polytopes obtained from the model $f$. The key idea used in this function is that each polytope is associated with a unique ReLU activation code. Verification for step $(1)$ involves computing the ReLU activation code for $\x$ using the committed weights $\textsf{com}_\W$ and step $(2)$ involves deriving the corresponding polytope for this activation code from $\textsf{com}_\W$. 
    
    \textit{2. Verifying distance to facets (Alg. \ref{alg:verify:distance}).} During its course \textit{GeoCert} computes distance between $\x$ and various facets. Hence, the verifier needs to check the correctness of these distance computations. As discussed in the preceding section, we compute a lower bound of the exact distance using projections, which can be efficiently proved under ZKPs. 

\textit{3. Verifying neighboring polytopes (Alg. \ref{alg:verify:neighbor}).} In each iteration \textit{GeoCert} visits a neighboring polytope adjacent to the current one; the two polytopes share the facet that was popped in the current iteration. Verifying neighborhood entails checking that the visited polytope indeed $(1)$ comes from the model $\m$, and $(2)$ shares the popped facet. The key idea used here is that two neighboring polytopes differ in a single ReLU activation corresponding to the shared facet (Fact \ref{fact:ReLU}). Specifically, the prover retrieves the representative point corresponding to the visited polytope and computes its ReLU activation code, $R'$, using the committed weights $\textsf{com}_\W$. Next, it computes the polytope corresponding to $R'$ from $\textsf{com}_\W$ to prove that it is obtained from the model $\m$. This is followed by showing that the hamming  distance between $R'$ and $R$ is one, where $R$ is the activation code for the current polytope. Finally, the prover shows that the current facet is common to both the polytopes.

 \textit{4. Verifying boundary facet (Alg. \ref{alg:verify:boundary}).} The termination condition of \textit{GeoCert} checks whether the current facet is a boundary facet or not; we verify this in \name~as follows. Let $R$ denote the activation code for the current polytope $\P$ and let $f_R(x)=W_{R}x+b_R$ represent the linear function associated with $R$. For the ease of exposition, let $\m$ be a binary classifier. In other words, $f_R(x)$ is the input to the softmax function in the final layer of $\m$ (i.e., logits) for all data points $x \in \P$. The key idea for verification is that \textit{iff} $x$ lies on a boundary facet, $f_R(x)$ has the \textit{same} value for both the logits. For verifying this computation, we rely on the pre-computed representative point of a facet. Specifically, the prover retrieves the representative point $z$ for the current facet $\F=\{x|Ax\leq b\}$. First, it proves  that $z$ lies on $\F$ by showing $A z\leq b$ holds. Next, the prover computes $f_R$ (i.e., the weights $W_R$ and $b_R$) from the committed weights using $R$ and tests the equality of both the logits in $f_R(z)$.

 \textit{5. Verify order of facet traversal (Alg. \ref{alg:verify:PQ}).} The order in which the facets are traversed needs to be verified -- this is equivalent to checking the functionality of the priority queue in \textit{GeoCert}. Standard ZKP tools are built for verifying mathematical computations (expressed as an arithmetic or Boolean circuit) and do not have built-in support for data structures, such as priority queues. We overcome this challenge by leveraging the following key idea -- correctness of the priority queue can be verified by checking that the next traversed facet is indeed the one with the shortest distance. 

\noindent\textbf{Additional optimizations.} We identify certain computations in the above algorithms that can performed offline. Specifically, in \textit{VerifyNeighbor} the proof of correctness for polytope construction using representative points can be generated offline. Further, in \textit{VerifyBoundary} proof for computation of the linear function $f_R$ can also be generated offline. This leads to a significant reduction in the cost of the online proof generation (see Sec. \ref{subsec:perf}). 

End-to-end verification mechanism is presented in Alg. \ref{alg:name}.  In the final step, the prover has to generate an additional proof that the reported certificate of fairness corresponds to the smallest value among all the lower bounds obtained for each element of $domain(\S)$ (\textit{VerifyMin}
, Alg. \ref{alg:verify:min}). Additionally, the prover also needs to prove integrity of the inference, i.e., $y=\m(\x)$. For this, after computing the linear function $f_{R_{x^*}}(\x)$ using the committed weights $\textsf{com}_\W$ (where $R_{\x}$ is the activation code for $\x$) we need to additionally prove that the label corresponds to the logit with the highest score (Alg. \ref{alg:verify:inference}, \textit{VerifyInference}).

Next, we present our security guarantee.
\begin{theorem}(Informal)
Given a model $\m$ and a data point $\x$, \name~provides the prediction $\m(\x)$ and a lower bound $\epsilon_{LB}$ on the individual fairness parameter for $\x$ without leaking anything, except the number of total facets traversed,  about the model $\m$.\label{thm:name}
\end{theorem}
\textit{Proof Sketch.} Proof of the above theorem follows directly from the properties of zero-knowledge proofs and theorems in App. \ref{app:proof}. The formal guarantee and detailed proof is presented in App. \ref{app:proof}.
\section{Evaluation}\label{sec:evaluation}

In this section we evaluate the performance of \name empirically. Specifically, we ask the following questions: \begin{compactenum}
\item Can our fairness certification mechanism distinguish between fair and unfair models?
\item Is \name practically feasible, in terms of time and communication costs?
\end{compactenum}

\textbf{Datasets.} We use three standard fairness benchmarks. \textit{Adult \cite{Adult}} is a dataset for income classification, where we select \textit{gender} (male/female) as the sensitive feature. \textit{Default Credit~\cite{credit}} is a dataset for predicting loan defaults, with \textit{gender} (male/female) as the chosen sensitive feature. Finally, \textit{German Credit~\cite{German}} is a loan application dataset, where  \textit{Foreign Worker} (yes/no) is used as the sensitive feature.

\textbf{Configuration.} We train fully-connected ReLU networks with stochastic gradient descent in PyTorch. Our networks have 2 hidden layers with different sizes including $(4,2)$, $(2,4)$ and $(8,2)$. All the dataset features are standardized ~\cite{Standarization}. \name is implemented using the Gnark~\cite{gnark-v0.9.0} zk-SNARK library in GoLang. We run all our code for \name without any multithreading or parallelism, on an Intel-i9 CPU chip with 28 cores.

\begin{figure*}[hbt!]
    \centering
    \begin{minipage}{0.33\linewidth}
        \centering
        \includegraphics[width=\linewidth]{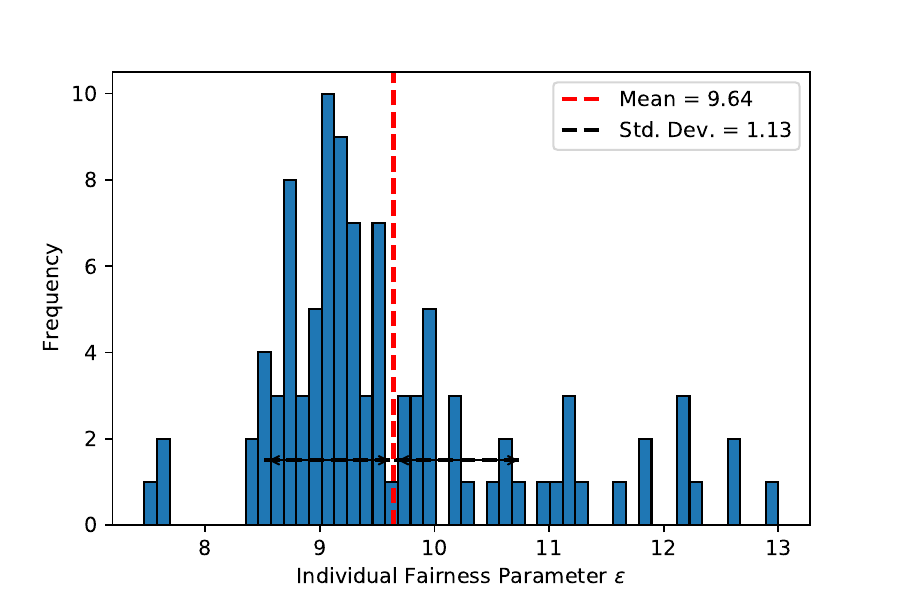}
        \caption*{\textit{Credit, Fair}}
        \label{fig:credit:fair}
    \end{minipage}\hfill
    \begin{minipage}{0.33\linewidth}
        \centering
        \includegraphics[width=\linewidth]{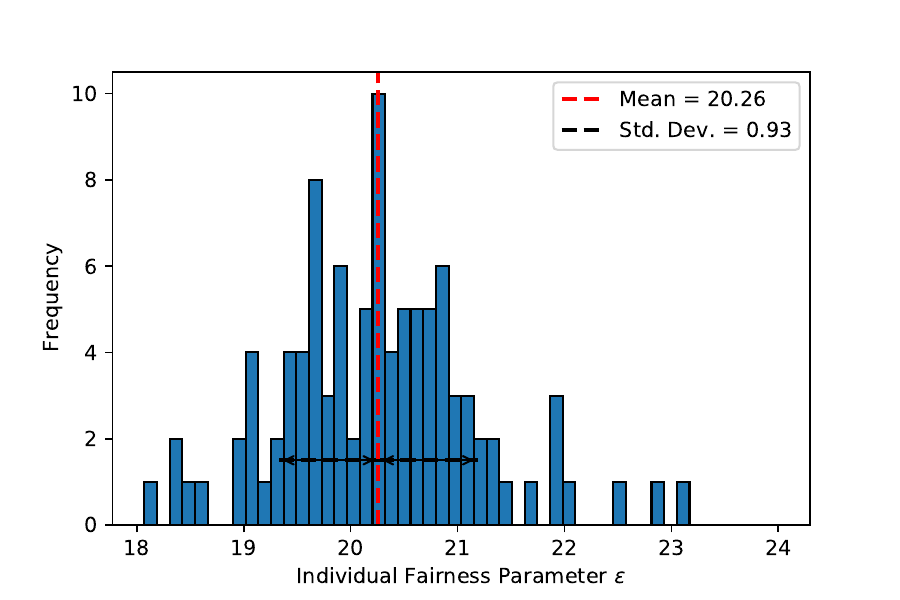}
        \caption*{\textit{Adult, Fair}}
        \label{fig:adult:fair}
    \end{minipage}\hfill
    \begin{minipage}{0.33\linewidth}
            \centering
            \includegraphics[width=\linewidth]{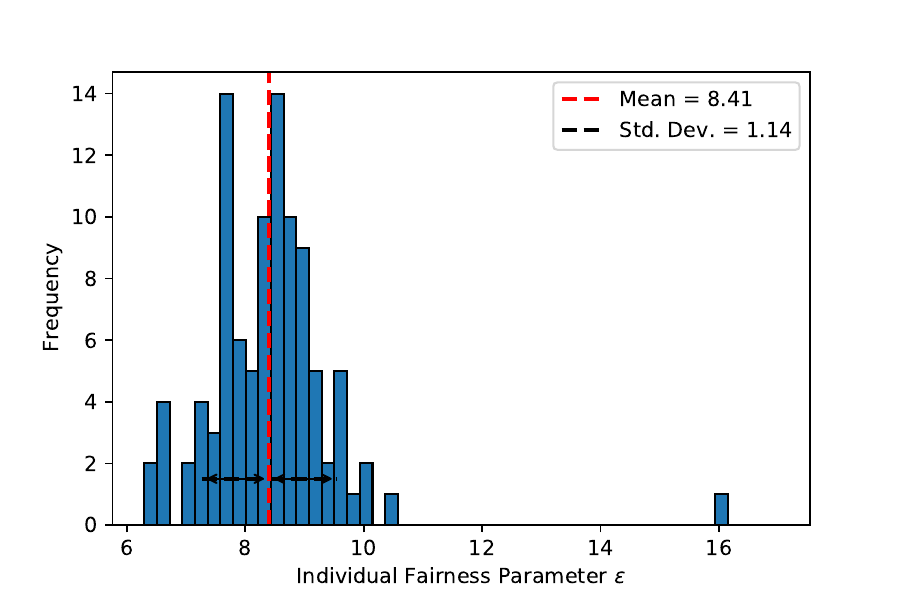}
            \caption*{\textit{German, Fair}}
            \label{fig:adult:fair}
        \end{minipage}\hfill\\
    
    \begin{minipage}{0.33\linewidth}
        \centering
        \includegraphics[width=\linewidth]{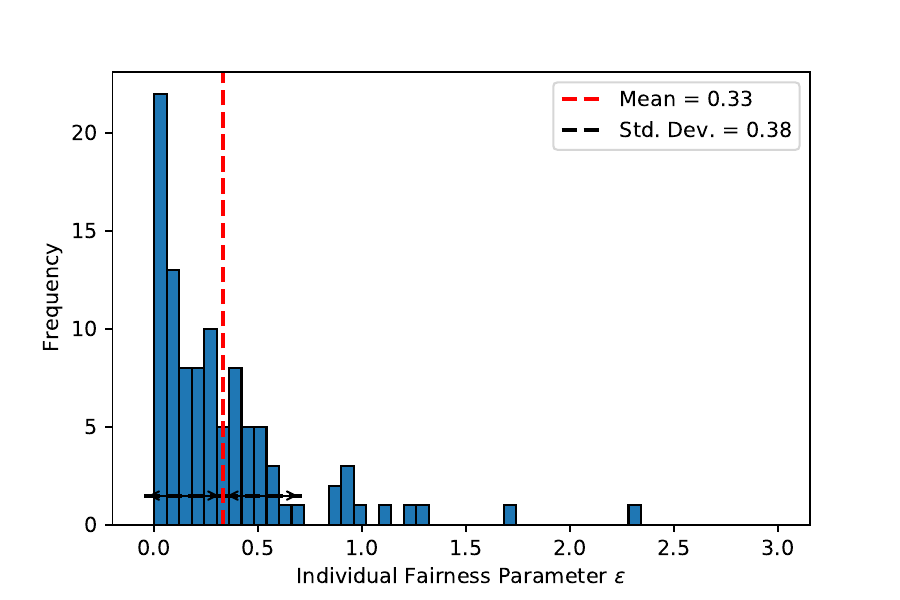}
        \caption*{\textit{Credit, Unfair}}
        \label{fig:credit:unfair}
    \end{minipage}\hfill
    \begin{minipage}{0.33\linewidth}
        \centering
        \includegraphics[width=\linewidth]{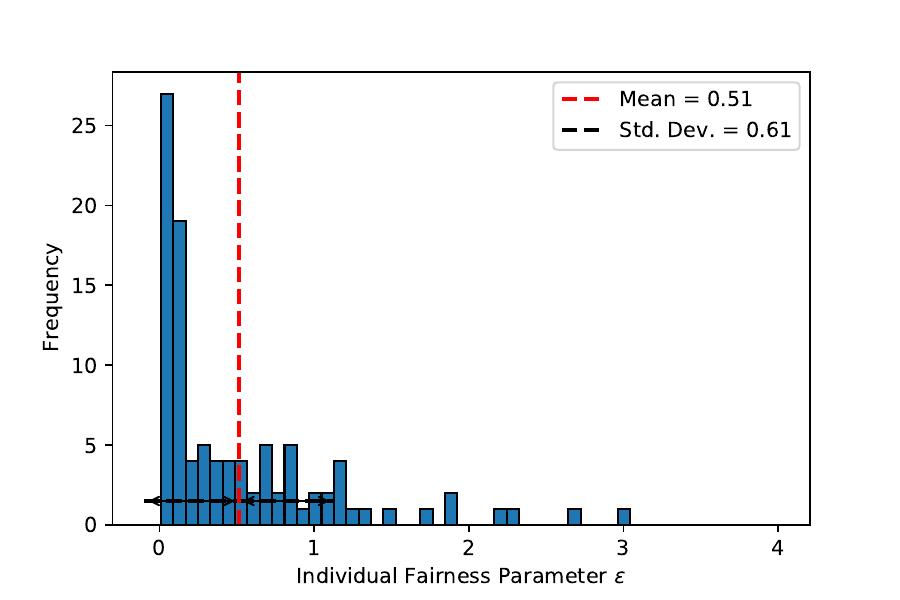}
        \caption*{\textit{Adult, Unfair}}
        \label{fig:adult:unfair}
    \end{minipage}\hfill
    \begin{minipage}{0.33\linewidth}
        \centering
        \includegraphics[width=\linewidth]{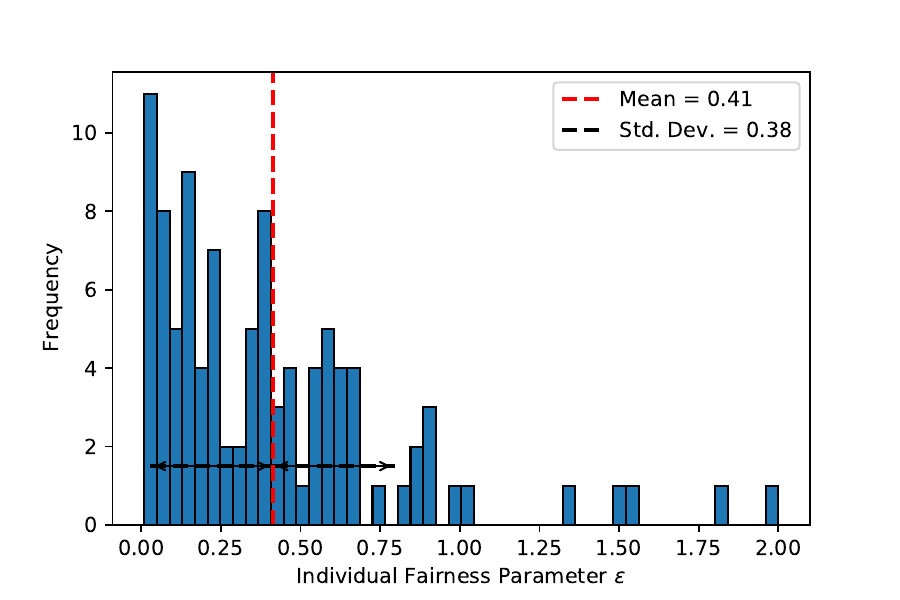}
        \caption*{\textit{German, Unfair}}
        \label{fig:adult:fair}
    \end{minipage}\hfill
    \caption{Histogram of fairness parameter $\epsilon$ for fair \& unfair models, model size = (4,2). $\epsilon$ values are higher than those for unfair models.}
    \label{fig:fair-unfair}
\end{figure*}

\paragraph{Model Fairness}

We first evaluate if our certification mechanism can distinguish between fair and unfair models. Prior work~\citep{islam2021can} has shown that overfitting leads to more unfair models while regularization encourages fairness. Thus, to obtain models with different fairness, we vary regularization by changing the weight decay parameter in PyTorch. Then we randomly sample $100$ test data points as input queries and find the fairness parameter $\epsilon$ for both types of models on these queries.

As demonstrated in Fig. ~\ref{fig:fair-unfair}, the unfair models have a lower $\epsilon$ than the corresponding fair models. This consistent difference in $\epsilon$ values across different model sizes and datasets shows that our certification mechanism can indeed distinguish between fair and unfair models. Results for other models are included in App. \ref{app:evaluation}. 

\paragraph{Performance of \name}\label{subsec:perf}

\begin{figure*}[hbt!]
    \centering
    \begin{minipage}{0.32\linewidth}
        \centering
        \includegraphics[width=\linewidth]{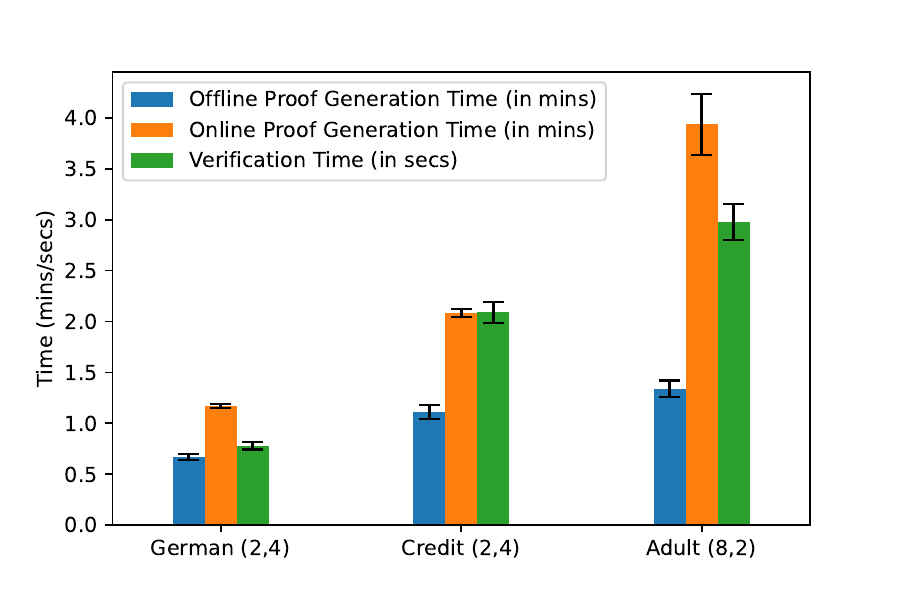}
        \caption*{(a)}\label{fig:timings}
    \end{minipage}\hfill
    \begin{minipage}{0.355\linewidth}
        \centering
        \includegraphics[width=\linewidth]{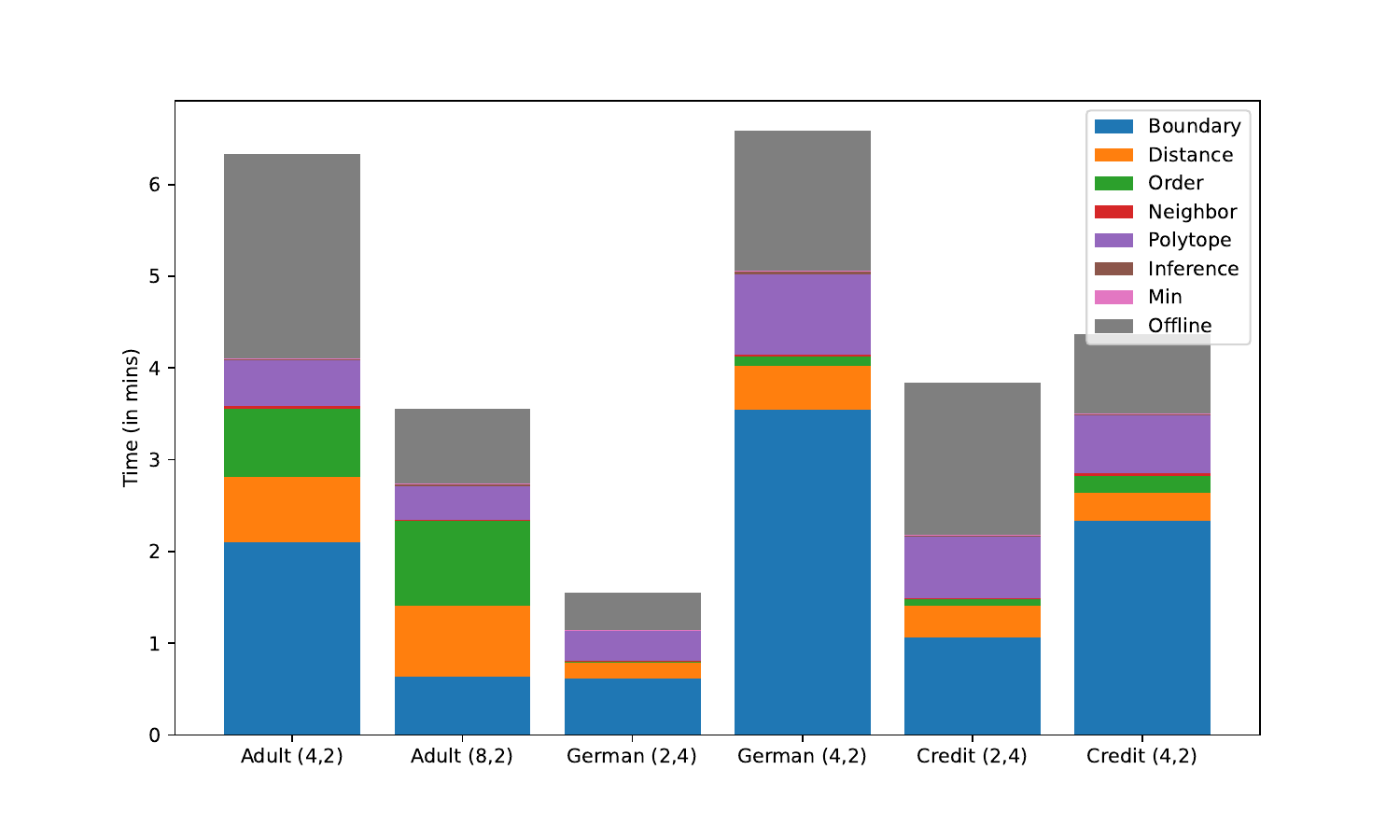}
        \caption*{(b)}\label{fig:breakdown}
    \end{minipage}\hfill
    \begin{minipage}{0.32\linewidth}
        \centering
        \includegraphics[width=\linewidth]{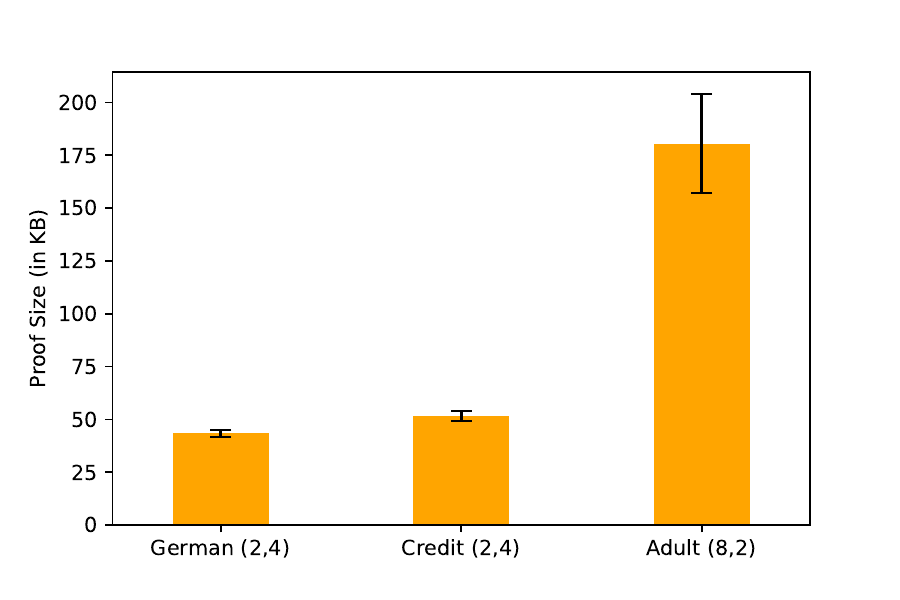}
        \caption*{(c)}
        \label{fig:psize}
    \end{minipage}\hfill
    \caption{(a) Proof Generation (in mins) and Verification times (in secs) for different models. Offline computations are done in the initial setup phase while Online computations are done for every new query. Verification is only done online, for every query. (b) Breakdown of the proof generation time (in mins) for the data point with the median time. (c) Total Proof Size (in KB) for various models. This includes the proof generated during both online and offline phases.}
    \label{fig:pvtime_psize_brkdwn}
\end{figure*}

Since computation is a known bottleneck in ZKPs, we next investigate the overhead of \name in terms of time and communication costs. All reported numbers are averages over a $100$ random test points. 

Fig.~\ref{fig:pvtime_psize_brkdwn} (a) shows the proof generation costs for various models. Note that the proof generation time varies with the models, due to its dependence on the number of traversed facets\footnote{As mentioned in Thm. \ref{thm:name}, this information is leaked by \name.} which in turn depends on the model and query. On average, the adult model has a larger number of traversed facets than others as shown in Table \ref{tab:proof} in App. \ref{app:evaluation}, leading to a higher proof generation time. We also observe that performing some computations in an offline phase results in significant reductions in the online time cost, the largest being 1.74$\times$. See Table \ref{tab:proof} and Fig.\ref{fig:proof} in App.\ref{app:evaluation} for details.


We also breakdown the overall proof generation time in terms of different sub-functionalities. We report this breakdown for the query with the median proof generation cost, in Fig. \ref{fig:pvtime_psize_brkdwn} (b). As shown in the figure, \textit{VerifyBoundary} is the costliest sub-function for all the models; this is so since it is executed in every iteration (every time a facet is popped) and involves costly non-linear comparison operations (see Alg. \ref{alg:verify:boundary}). Other functionalities that are also executed multiple times based on number of traversed facets but are not as expensive include \textit{VerifyNeighbor}, \textit{VerifyDistance} and \textit{VerifyOrder} (see Alg. \ref{alg:verify:neighbor}, \ref{alg:verify:distance}, \ref{alg:verify:PQ}). The least time is taken by \textit{VerifyMin} which basically finds the minimum in a list; this is so since the function is straight-forward and is ran only once per query (see Alg. \ref{alg:verify:min}).


We also report the average verification times - time for checking the validity of the proof by the verifier - in Fig.~\ref{fig:pvtime_psize_brkdwn} (a). Note that the verification costs are orders of magnitude lower (in seconds) than the proof generation costs (in minutes) for all models; as is standard in ZKPs. Fig.\ref{fig:pvtime_psize_brkdwn} (c) reports the communication overheads, i.e. size of the generated proofs. The proof size is very small, only certain kilobytes. Low verification time and communication cost is advantageous since it implies quick real-time verification which does not require complex machinery at the customer end. For detailed results on all models, refer to Fig. \ref{fig:verification} and Fig. \ref{fig:communication} in App. \ref{app:evaluation}.


\textbf{Discussion on Scalability} For very large models, the number of traversed facets can be huge and running \name on them may not be practically feasible anymore. In such cases, one solution can be just verifying the fairness of the final layers. We leave this exploration to future work.
\section{Related Work}\label{sec:related work main}
\textbf{Verifiable fairness with cryptography.} Most of the prior work on verifying fairness \textit{while maintaining model confidentiality}~\cite{pentyala2022privfair,BlindJustice,Toreini2023VerifiableFP,Segal21,Park22} has approached the problem in the third-party auditor setting. The closest to ours is a recent work by \cite{confidant}, which proposed a fairness-aware training pipeline for \textit{decision trees} that allows the model owner to cryptographically prove that \textit{the learning
algorithm} used to train the model was fair by design. In contrast, we focus on neural networks and issue a fairness certificate by simply inspecting the model weights \textit{post-training}. Our system \name and certification mechanism is completely agnostic of the training pipeline.

Another line of work has been using cryptographic primitives to verify other properties (rather than fairness) of an ML model while maintaining model confidentiality -- \cite{ZKDT, Liu2021zkCNNZK} focus on accuracy and inference, while ~\cite{VeriML, PoT, sun2023zkdl} focus on the training process.

A separate line of work uses formal verification approaches for verifying the fairness of a model ~\cite{FairSquare,Bastani2019,Urban2020,Ghosh2020JusticiaAS,Fairify}.
However, these works focus on certification in the plain text, i.e., they do not preserve model confidentiality and do not involve any cryptography.

\textbf{Fairness Certification Mechanisms.} Prior work on certification mechanisms for fairness can be broadly classified into three categories. The first line of work frames the certification problem as an optimization program \cite{john2020verifying,Benussi2022IndividualFG,kang2022certifying}. The second line of research has leveraged the connection between robustness and fairness, and proposed fairness-aware training mechanisms akin to adversarial training \cite{Ruoss2020,Yurochkin2020Training,khedr2022certifair,smoothing,doherty2023individual}. In contrast to both, we focus on \textit{local} IF specifically for neural networks and use an iterative algorithm rather than solving a complex optimization problem and are completely agnostic of the training pipeline.

The final line of work is based on black-box query access learning theoretic approaches~\cite{yadav2022learningtheoretic,yan2022active,AVOIR}. Contrary to our work, these approaches however are replete with problems arising from the usage of a reference dataset~\cite{fukuchi2019faking, confidant}, the need for a trust third-party, and lack of guarantees of model uniformity.

See App. Sec. \ref{sec:related work appendix} for a further discussion on related works.
\section{Conclusion}\label{sec:conclusion}
In this paper we proposed \name -- a protocol enabling model owners to issue publicly verifiable certificates while ensuring model uniformity and confidentiality. Our experiments demonstrate the practical feasibility of  \name~ for small neural networks and tabular data. While our work is grounded in fairness and societal applications, we believe that ZKPs are a general-purpose tool and can be a promising solution for overcoming problems arising out of the need for model confidentiality in other areas/applications as well. We call for further research in this direction.

\paragraph{Acknowledgements} KC and CY acknowledge the following grants for research support - CNS 1804829, NSF under 2241100, NSF under 2217058, ARO-MURI W911NF2110317 and ONR under N00014-20-1-2334. DB acknowledges support from NSF, DARPA, the Simons Foundation, UBRI, and NTT Research. Opinions, findings, and conclusions or recommendations expressed in this material are those of the authors and do not necessarily reflect the views of DARPA. CY and ARC would like to thank Matt Jordan for his help with understanding GeoCert and its implementation.

\bibliographystyle{iclr2024_conference.bst}
\bibliography{references.bib}

\begin{thebibliography}{69}
\providecommand{\natexlab}[1]{#1}
\providecommand{\url}[1]{\texttt{#1}}
\expandafter\ifx\csname urlstyle\endcsname\relax
  \providecommand{\doi}[1]{doi: #1}\else
  \providecommand{\doi}{doi: \begingroup \urlstyle{rm}\Url}\fi

\bibitem[Sta()]{Standarization}
Standarization.
\newblock \url{https://scikit-learn.org/stable/modules/generated/sklearn.preprocessing.StandardScaler.html}.

\bibitem[VI2(2023)]{VI2}
Zator: Verified inference of a 512-layer neural network using recursive snarks.
\newblock \url{https://github.com/lyronctk/zator/tree/main}, 2023.

\bibitem[Albarghouthi et~al.(2017)Albarghouthi, D'Antoni, Drews, and Nori]{FairSquare}
Aws Albarghouthi, Loris D'Antoni, Samuel Drews, and Aditya~V. Nori.
\newblock Fairsquare: Probabilistic verification of program fairness.
\newblock \emph{Proc. ACM Program. Lang.}, 1\penalty0 (OOPSLA), oct 2017.
\newblock \doi{10.1145/3133904}.
\newblock URL \url{https://doi.org/10.1145/3133904}.

\bibitem[Angwin et~al.(2016)Angwin, Larson, Mattu, and Kirchner]{unfair1}
Julia Angwin, Jeff Larson, Surya Mattu, and Lauren Kirchner.
\newblock Machine bias.
\newblock \url{https://www.propublica.org/article/machine-bias-risk-assessments-in-criminal-sentencing}, 2016.

\bibitem[Barocas et~al.(2019)Barocas, Hardt, and Narayanan]{barocas-hardt-narayanan}
Solon Barocas, Moritz Hardt, and Arvind Narayanan.
\newblock \emph{Fairness and Machine Learning: Limitations and Opportunities}.
\newblock fairmlbook.org, 2019.
\newblock \url{http://www.fairmlbook.org}.

\bibitem[Bastani et~al.(2019)Bastani, Zhang, and Solar-Lezama]{Bastani2019}
Osbert Bastani, Xin Zhang, and Armando Solar-Lezama.
\newblock Probabilistic verification of fairness properties via concentration.
\newblock \emph{Proc. ACM Program. Lang.}, 3\penalty0 (OOPSLA), oct 2019.
\newblock \doi{10.1145/3360544}.
\newblock URL \url{https://doi.org/10.1145/3360544}.

\bibitem[Becker \& Kohavi(1996)Becker and Kohavi]{Adult}
Barry Becker and Ronny Kohavi.
\newblock {Adult}.
\newblock UCI Machine Learning Repository, 1996.
\newblock {DOI}: https://doi.org/10.24432/C5XW20.

\bibitem[Benussi et~al.(2022)Benussi, Patan{\'e}, Wicker, Laurenti, of~Oxford, and Delft]{Benussi2022IndividualFG}
Elias Benussi, Andrea Patan{\'e}, Matthew Wicker, Luca Laurenti, Marta Kwiatkowska~University of~Oxford, and Tu~Delft.
\newblock Individual fairness guarantees for neural networks.
\newblock In \emph{International Joint Conference on Artificial Intelligence}, 2022.
\newblock URL \url{https://api.semanticscholar.org/CorpusID:248722046}.

\bibitem[Bertrand \& Mullainathan(2004)Bertrand and Mullainathan]{IF}
Marianne Bertrand and Sendhil Mullainathan.
\newblock Are emily and greg more employable than lakisha and jamal? a field experiment on labor market discrimination.
\newblock \emph{American Economic Review}, 94\penalty0 (4):\penalty0 991--1013, September 2004.
\newblock \doi{10.1257/0002828042002561}.
\newblock URL \url{https://www.aeaweb.org/articles?id=10.1257/0002828042002561}.

\bibitem[Biswas \& Rajan(2023)Biswas and Rajan]{Fairify}
Sumon Biswas and Hridesh Rajan.
\newblock Fairify: Fairness verification of neural networks.
\newblock In \emph{2023 IEEE/ACM 45th International Conference on Software Engineering (ICSE)}, pp.\  1546--1558, 2023.
\newblock \doi{10.1109/ICSE48619.2023.00134}.

\bibitem[Boemer et~al.(2020)Boemer, Cammarota, Demmler, Schneider, and Yalame]{boemer2020mp2ml}
Fabian Boemer, Rosario Cammarota, Daniel Demmler, Thomas Schneider, and Hossein Yalame.
\newblock Mp2ml: A mixed-protocol machine learning framework for private inference.
\newblock In \emph{Proceedings of the 15th international conference on availability, reliability and security}, pp.\  1--10, 2020.

\bibitem[Botrel et~al.(2023)Botrel, Piellard, Housni, Kubjas, and Tabaie]{gnark-v0.9.0}
Gautam Botrel, Thomas Piellard, Youssef~El Housni, Ivo Kubjas, and Arya Tabaie.
\newblock Consensys/gnark: v0.9.0, February 2023.
\newblock URL \url{https://doi.org/10.5281/zenodo.5819104}.

\bibitem[Brennan et~al.(2009)Brennan, Dieterich, and Ehret]{crimeprediction}
Tim Brennan, William Dieterich, and Beate Ehret.
\newblock Evaluating the predictive validity of the compas risk and needs assessment system.
\newblock \emph{Criminal Justice and Behavior}, 36\penalty0 (1):\penalty0 21--40, 2009.
\newblock \doi{10.1177/0093854808326545}.
\newblock URL \url{https://doi.org/10.1177/0093854808326545}.

\bibitem[Casper et~al.(2024)Casper, Ezell, Siegmann, Kolt, Curtis, Bucknall, Haupt, Wei, Scheurer, Hobbhahn, et~al.]{casper2024black}
Stephen Casper, Carson Ezell, Charlotte Siegmann, Noam Kolt, Taylor~Lynn Curtis, Benjamin Bucknall, Andreas Haupt, Kevin Wei, J{\'e}r{\'e}my Scheurer, Marius Hobbhahn, et~al.
\newblock Black-box access is insufficient for rigorous ai audits.
\newblock In \emph{The 2024 ACM Conference on Fairness, Accountability, and Transparency}, pp.\  2254--2272, 2024.

\bibitem[Croce et~al.(2019)Croce, Andriushchenko, and Hein]{croce2019provable}
Francesco Croce, Maksym Andriushchenko, and Matthias Hein.
\newblock Provable robustness of relu networks via maximization of linear regions.
\newblock In \emph{the 22nd International Conference on Artificial Intelligence and Statistics}, pp.\  2057--2066. PMLR, 2019.

\bibitem[Dastin(October 2018)]{resume}
J~Dastin.
\newblock Amazon scraps secret ai recruiting tool that showed bias against women, October 2018.

\bibitem[Datta et~al.(2014)Datta, Tschantz, and Datta]{adprediction}
Amit Datta, Michael~Carl Tschantz, and Anupam Datta.
\newblock Automated experiments on ad privacy settings: A tale of opacity, choice, and discrimination.
\newblock \emph{ArXiv}, abs/1408.6491, 2014.
\newblock URL \url{https://api.semanticscholar.org/CorpusID:6817607}.

\bibitem[Doherty et~al.(2023)Doherty, Wicker, Laurenti, and Patane]{doherty2023individual}
Alice Doherty, Matthew Wicker, Luca Laurenti, and Andrea Patane.
\newblock Individual fairness in bayesian neural networks, 2023.

\bibitem[Dwork \& Minow(2022)Dwork and Minow]{dwork2022distrust}
Cynthia Dwork and Martha Minow.
\newblock Distrust of artificial intelligence: Sources \& responses from computer science \& law.
\newblock \emph{Daedalus}, 151\penalty0 (2):\penalty0 309--321, 2022.

\bibitem[Dwork et~al.(2012)Dwork, Hardt, Pitassi, Reingold, and Zemel]{Dwork12}
Cynthia Dwork, Moritz Hardt, Toniann Pitassi, Omer Reingold, and Richard Zemel.
\newblock Fairness through awareness.
\newblock In \emph{Proceedings of the 3rd Innovations in Theoretical Computer Science Conference}, ITCS '12, pp.\  214–226, New York, NY, USA, 2012. Association for Computing Machinery.
\newblock ISBN 9781450311151.
\newblock \doi{10.1145/2090236.2090255}.
\newblock URL \url{https://doi.org/10.1145/2090236.2090255}.

\bibitem[Feng et~al.(2021)Feng, Qin, Zhang, Ding, and Chu]{Zen}
Boyuan Feng, Lianke Qin, Zhenfei Zhang, Yufei Ding, and Shumo Chu.
\newblock Zen: Efficient zero-knowledge proofs for neural networks.
\newblock \emph{IACR Cryptol. ePrint Arch.}, 2021:\penalty0 87, 2021.
\newblock URL \url{https://api.semanticscholar.org/CorpusID:231731893}.

\bibitem[Fukuchi et~al.(2019)Fukuchi, Hara, and Maehara]{fukuchi2019faking}
Kazuto Fukuchi, Satoshi Hara, and Takanori Maehara.
\newblock Faking fairness via stealthily biased sampling, 2019.

\bibitem[Garg et~al.(2023)Garg, Goel, Jha, Mahloujifar, Mahmoody, Policharla, and Wang]{PoT}
Sanjam Garg, Aarushi Goel, Somesh Jha, Saeed Mahloujifar, Mohammad Mahmoody, Guru-Vamsi Policharla, and Mingyuan Wang.
\newblock Experimenting with zero-knowledge proofs of training.
\newblock Cryptology ePrint Archive, Paper 2023/1345, 2023.
\newblock URL \url{https://eprint.iacr.org/2023/1345}.
\newblock \url{https://eprint.iacr.org/2023/1345}.

\bibitem[Ghosh et~al.(2020)Ghosh, Basu, and Meel]{Ghosh2020JusticiaAS}
Bishwamittra Ghosh, D.~Basu, and Kuldeep~S. Meel.
\newblock Justicia: A stochastic sat approach to formally verify fairness.
\newblock In \emph{AAAI Conference on Artificial Intelligence}, 2020.
\newblock URL \url{https://api.semanticscholar.org/CorpusID:221655566}.

\bibitem[Goldreich et~al.(1991)Goldreich, Micali, and Wigderson]{GMW}
Oded Goldreich, Silvio Micali, and Avi Wigderson.
\newblock Proofs that yield nothing but their validity or all languages in np have zero-knowledge proof systems.
\newblock \emph{J. ACM}, 38\penalty0 (3):\penalty0 690–728, jul 1991.
\newblock ISSN 0004-5411.
\newblock \doi{10.1145/116825.116852}.
\newblock URL \url{https://doi.org/10.1145/116825.116852}.

\bibitem[Goldwasser et~al.(1985)Goldwasser, Micali, and Rackoff]{GMR}
S~Goldwasser, S~Micali, and C~Rackoff.
\newblock The knowledge complexity of interactive proof-systems.
\newblock In \emph{Proceedings of the Seventeenth Annual ACM Symposium on Theory of Computing}, STOC '85, pp.\  291–304, New York, NY, USA, 1985. Association for Computing Machinery.
\newblock ISBN 0897911512.
\newblock \doi{10.1145/22145.22178}.
\newblock URL \url{https://doi.org/10.1145/22145.22178}.

\bibitem[Gupta et~al.(2023)Gupta, Jawalkar, Mukherjee, Chandran, Gupta, Panwar, and Sharma]{gupta2023sigma}
Kanav Gupta, Neha Jawalkar, Ananta Mukherjee, Nishanth Chandran, Divya Gupta, Ashish Panwar, and Rahul Sharma.
\newblock Sigma: secure gpt inference with function secret sharing.
\newblock \emph{Cryptology ePrint Archive}, 2023.

\bibitem[Hamman et~al.(2023)Hamman, Chen, and Dutta]{hamman2023can}
Faisal Hamman, Jiahao Chen, and Sanghamitra Dutta.
\newblock Can querying for bias leak protected attributes? achieving privacy with smooth sensitivity.
\newblock In \emph{Proceedings of the 2023 ACM Conference on Fairness, Accountability, and Transparency}, pp.\  1358--1368, 2023.

\bibitem[Hanin \& Rolnick(2019)Hanin and Rolnick]{hanin2019deep}
Boris Hanin and David Rolnick.
\newblock Deep relu networks have surprisingly few activation patterns.
\newblock \emph{Advances in neural information processing systems}, 32, 2019.

\bibitem[Hofmann(1994)]{German}
Hans Hofmann.
\newblock {Statlog (German Credit Data)}.
\newblock UCI Machine Learning Repository, 1994.
\newblock {DOI}: https://doi.org/10.24432/C5NC77.

\bibitem[Islam et~al.(2021)Islam, Pan, and Foulds]{islam2021can}
Rashidul Islam, Shimei Pan, and James~R Foulds.
\newblock Can we obtain fairness for free?
\newblock In \emph{Proceedings of the 2021 AAAI/ACM Conference on AI, Ethics, and Society}, pp.\  586--596, 2021.

\bibitem[John et~al.(2020)John, Vijaykeerthy, and Saha]{john2020verifying}
Philips~George John, Deepak Vijaykeerthy, and Diptikalyan Saha.
\newblock Verifying individual fairness in machine learning models, 2020.

\bibitem[Jordan et~al.(2019)Jordan, Lewis, and Dimakis]{Geocert}
Matt Jordan, Justin Lewis, and Alexandros~G. Dimakis.
\newblock \emph{Provable Certificates for Adversarial Examples: Fitting a Ball in the Union of Polytopes}.
\newblock Curran Associates Inc., Red Hook, NY, USA, 2019.

\bibitem[Juvekar et~al.(2018)Juvekar, Vaikuntanathan, and Chandrakasan]{juvekar2018gazelle}
Chiraag Juvekar, Vinod Vaikuntanathan, and Anantha Chandrakasan.
\newblock $\{$GAZELLE$\}$: A low latency framework for secure neural network inference.
\newblock In \emph{27th USENIX Security Symposium (USENIX Security 18)}, pp.\  1651--1669, 2018.

\bibitem[Kang et~al.(2022{\natexlab{a}})Kang, Hashimoto, Stoica, and Sun]{kang2022scaling}
Daniel Kang, Tatsunori Hashimoto, Ion Stoica, and Yi~Sun.
\newblock Scaling up trustless dnn inference with zero-knowledge proofs, 2022{\natexlab{a}}.

\bibitem[Kang et~al.(2022{\natexlab{b}})Kang, Li, Weber, Liu, Zhang, and Li]{kang2022certifying}
Mintong Kang, Linyi Li, Maurice Weber, Yang Liu, Ce~Zhang, and Bo~Li.
\newblock Certifying some distributional fairness with subpopulation decomposition, 2022{\natexlab{b}}.

\bibitem[Khandani et~al.(2010)Khandani, Kim, and Lo]{creditprediction}
Amir~E. Khandani, Adlar~J. Kim, and Andrew~W. Lo.
\newblock Consumer credit-risk models via machine-learning algorithms.
\newblock \emph{Journal of Banking \& Finance}, 34\penalty0 (11):\penalty0 2767--2787, 2010.
\newblock ISSN 0378-4266.
\newblock \doi{https://doi.org/10.1016/j.jbankfin.2010.06.001}.
\newblock URL \url{https://www.sciencedirect.com/science/article/pii/S0378426610002372}.

\bibitem[Khedr \& Shoukry(2022)Khedr and Shoukry]{khedr2022certifair}
Haitham Khedr and Yasser Shoukry.
\newblock Certifair: A framework for certified global fairness of neural networks, 2022.

\bibitem[Kilbertus et~al.(2018)Kilbertus, Gascon, Kusner, Veale, Gummadi, and Weller]{BlindJustice}
Niki Kilbertus, Adria Gascon, Matt Kusner, Michael Veale, Krishna Gummadi, and Adrian Weller.
\newblock Blind justice: Fairness with encrypted sensitive attributes.
\newblock In Jennifer Dy and Andreas Krause (eds.), \emph{Proceedings of the 35th International Conference on Machine Learning}, volume~80 of \emph{Proceedings of Machine Learning Research}, pp.\  2630--2639. PMLR, 10--15 Jul 2018.
\newblock URL \url{https://proceedings.mlr.press/v80/kilbertus18a.html}.

\bibitem[Lee et~al.(2020)Lee, Ko, Kim, and Oh]{vCNN}
Seunghwan Lee, Hankyung Ko, Jihye Kim, and Hyunok Oh.
\newblock vcnn: Verifiable convolutional neural network.
\newblock \emph{IACR Cryptol. ePrint Arch.}, 2020:\penalty0 584, 2020.
\newblock URL \url{https://api.semanticscholar.org/CorpusID:218895602}.

\bibitem[Liu et~al.(2017)Liu, Juuti, Lu, and Asokan]{liu2017oblivious}
Jian Liu, Mika Juuti, Yao Lu, and Nadarajah Asokan.
\newblock Oblivious neural network predictions via minionn transformations.
\newblock In \emph{Proceedings of the 2017 ACM SIGSAC conference on computer and communications security}, pp.\  619--631, 2017.

\bibitem[Liu et~al.(2021)Liu, Xie, and Zhang]{Liu2021zkCNNZK}
Tianyi Liu, Xiang Xie, and Yupeng Zhang.
\newblock zkcnn: Zero knowledge proofs for convolutional neural network predictions and accuracy.
\newblock \emph{Proceedings of the 2021 ACM SIGSAC Conference on Computer and Communications Security}, 2021.
\newblock URL \url{https://api.semanticscholar.org/CorpusID:235349006}.

\bibitem[Maneriker et~al.(2023)Maneriker, Burley, and Parthasarathy]{AVOIR}
Pranav Maneriker, Codi Burley, and Srinivasan Parthasarathy.
\newblock Online fairness auditing through iterative refinement.
\newblock In \emph{Proceedings of the 29th ACM SIGKDD Conference on Knowledge Discovery and Data Mining}, KDD '23, pp.\  1665–1676, New York, NY, USA, 2023. Association for Computing Machinery.
\newblock ISBN 9798400701030.
\newblock \doi{10.1145/3580305.3599454}.
\newblock URL \url{https://doi.org/10.1145/3580305.3599454}.

\bibitem[Mehrabi et~al.(2021)Mehrabi, Morstatter, Saxena, Lerman, and Galstyan]{survey}
Ninareh Mehrabi, Fred Morstatter, Nripsuta Saxena, Kristina Lerman, and Aram Galstyan.
\newblock A survey on bias and fairness in machine learning.
\newblock \emph{ACM Comput. Surv.}, 54\penalty0 (6), jul 2021.
\newblock ISSN 0360-0300.
\newblock \doi{10.1145/3457607}.
\newblock URL \url{https://doi.org/10.1145/3457607}.

\bibitem[Mohassel \& Rindal(2018)Mohassel and Rindal]{mohassel2018aby3}
Payman Mohassel and Peter Rindal.
\newblock Aby3: A mixed protocol framework for machine learning.
\newblock In \emph{Proceedings of the 2018 ACM SIGSAC conference on computer and communications security}, pp.\  35--52, 2018.

\bibitem[Mohassel \& Zhang(2017)Mohassel and Zhang]{mohassel2017secureml}
Payman Mohassel and Yupeng Zhang.
\newblock Secureml: A system for scalable privacy-preserving machine learning.
\newblock In \emph{2017 IEEE symposium on security and privacy (SP)}, pp.\  19--38. IEEE, 2017.

\bibitem[Park et~al.(2022)Park, Kim, and Lim]{Park22}
Saerom Park, Seongmin Kim, and Yeon-sup Lim.
\newblock Fairness audit of machine learning models with confidential computing.
\newblock In \emph{Proceedings of the ACM Web Conference 2022}, WWW '22, pp.\  3488–3499, New York, NY, USA, 2022. Association for Computing Machinery.
\newblock ISBN 9781450390965.
\newblock \doi{10.1145/3485447.3512244}.
\newblock URL \url{https://doi.org/10.1145/3485447.3512244}.

\bibitem[Pentyala et~al.(2022)Pentyala, Melanson, Cock, and Farnadi]{pentyala2022privfair}
Sikha Pentyala, David Melanson, Martine~De Cock, and Golnoosh Farnadi.
\newblock Privfair: a library for privacy-preserving fairness auditing, 2022.

\bibitem[Robinson et~al.(2019)Robinson, Rasheed, and San]{robinson2019dissecting}
Haakon Robinson, Adil Rasheed, and Omer San.
\newblock Dissecting deep neural networks.
\newblock \emph{arXiv preprint arXiv:1910.03879}, 2019.

\bibitem[Ruoss et~al.(2020)Ruoss, Balunovic, Fischer, and Vechev]{Ruoss2020}
Anian Ruoss, Mislav Balunovic, Marc Fischer, and Martin Vechev.
\newblock Learning certified individually fair representations.
\newblock In H.~Larochelle, M.~Ranzato, R.~Hadsell, M.F. Balcan, and H.~Lin (eds.), \emph{Advances in Neural Information Processing Systems}, volume~33, pp.\  7584--7596. Curran Associates, Inc., 2020.
\newblock URL \url{https://proceedings.neurips.cc/paper_files/paper/2020/file/55d491cf951b1b920900684d71419282-Paper.pdf}.

\bibitem[Segal et~al.(2021)Segal, Adi, Pinkas, Baum, Ganesh, and Keshet]{Segal21}
Shahar Segal, Yossi Adi, Benny Pinkas, Carsten Baum, Chaya Ganesh, and Joseph Keshet.
\newblock Fairness in the eyes of the data: Certifying machine-learning models.
\newblock In \emph{Proceedings of the 2021 AAAI/ACM Conference on AI, Ethics, and Society}, AIES '21, pp.\  926–935, New York, NY, USA, 2021. Association for Computing Machinery.
\newblock ISBN 9781450384735.
\newblock \doi{10.1145/3461702.3462554}.
\newblock URL \url{https://doi.org/10.1145/3461702.3462554}.

\bibitem[Serra et~al.(2018)Serra, Tjandraatmadja, and Ramalingam]{serra2018bounding}
Thiago Serra, Christian Tjandraatmadja, and Srikumar Ramalingam.
\newblock Bounding and counting linear regions of deep neural networks.
\newblock In \emph{International Conference on Machine Learning}, pp.\  4558--4566. PMLR, 2018.

\bibitem[Shamsabadi et~al.(2023)Shamsabadi, Wyllie, Franzese, Dullerud, Gambs, Papernot, Wang, and Weller]{confidant}
Ali~Shahin Shamsabadi, Sierra~Calanda Wyllie, Nicholas Franzese, Natalie Dullerud, Sébastien Gambs, Nicolas Papernot, Xiao Wang, and Adrian Weller.
\newblock Confidential proof of fair training of trees.
\newblock \emph{ICLR}, 2023.

\bibitem[Soares et~al.(2023)Soares, Yadav, Das, and Varshney]{soares2023keeping}
Ioana~Baldini Soares, Chhavi Yadav, Payel Das, and Kush Varshney.
\newblock Keeping up with the language models: Robustness-bias interplay in nli data and models.
\newblock In \emph{Annual Meeting of the Association for Computational Linguistics}, 2023.

\bibitem[Srinivasan et~al.(2019)Srinivasan, Akshayaram, and Ada]{srinivasan2019delphi}
Wenting~Zheng Srinivasan, PMRL Akshayaram, and Popa~Raluca Ada.
\newblock Delphi: A cryptographic inference service for neural networks.
\newblock In \emph{Proc. 29th USENIX Secur. Symp}, pp.\  2505--2522, 2019.

\bibitem[Sun \& Zhang(2023)Sun and Zhang]{sun2023zkdl}
Haochen Sun and Hongyang Zhang.
\newblock zkdl: Efficient zero-knowledge proofs of deep learning training, 2023.

\bibitem[Toreini et~al.(2023)Toreini, Mehrnezhad, and van Moorsel]{Toreini2023VerifiableFP}
Ehsan Toreini, Maryam Mehrnezhad, and Aad van Moorsel.
\newblock Verifiable fairness: Privacy-preserving computation of fairness for machine learning systems.
\newblock 2023.
\newblock URL \url{https://api.semanticscholar.org/CorpusID:261696588}.

\bibitem[Urban et~al.(2020)Urban, Christakis, W\"{u}stholz, and Zhang]{Urban2020}
Caterina Urban, Maria Christakis, Valentin W\"{u}stholz, and Fuyuan Zhang.
\newblock Perfectly parallel fairness certification of neural networks.
\newblock \emph{Proc. ACM Program. Lang.}, 4\penalty0 (OOPSLA), nov 2020.
\newblock \doi{10.1145/3428253}.
\newblock URL \url{https://doi.org/10.1145/3428253}.

\bibitem[Vigdor(November, 2019)]{AppleCard}
N~Vigdor.
\newblock Apple card investigated after gender discrimination complaints., November, 2019.

\bibitem[Wallarchive \& Schellmannarchive(June, 2021)Wallarchive and Schellmannarchive]{jobunfair}
Sheridan Wallarchive and Hilke Schellmannarchive.
\newblock Linkedin’s job-matching ai was biased. the company’s solution? more ai., June, 2021.

\bibitem[Weng et~al.(2023)Weng, Weng, Tang, Yang, Li, and Liu]{PvCNN}
Jiasi Weng, Jian Weng, Gui Tang, Anjia Yang, Ming Li, and Jia-Nan Liu.
\newblock Pvcnn: Privacy-preserving and verifiable convolutional neural network testing.
\newblock \emph{Trans. Info. For. Sec.}, 18:\penalty0 2218–2233, mar 2023.
\newblock ISSN 1556-6013.
\newblock \doi{10.1109/TIFS.2023.3262932}.
\newblock URL \url{https://doi.org/10.1109/TIFS.2023.3262932}.

\bibitem[Xu et~al.(2021)Xu, Vaughan, Chen, Zhang, and Sudjianto]{xu2021traversing}
Shaojie Xu, Joel Vaughan, Jie Chen, Aijun Zhang, and Agus Sudjianto.
\newblock Traversing the local polytopes of relu neural networks: A unified approach for network verification.
\newblock \emph{arXiv preprint arXiv:2111.08922}, 2021.

\bibitem[Yadav et~al.(2022)Yadav, Moshkovitz, and Chaudhuri]{yadav2022learningtheoretic}
Chhavi Yadav, Michal Moshkovitz, and Kamalika Chaudhuri.
\newblock A learning-theoretic framework for certified auditing with explanations, 2022.

\bibitem[Yan \& Zhang(2022)Yan and Zhang]{yan2022active}
Tom Yan and Chicheng Zhang.
\newblock Active fairness auditing, 2022.

\bibitem[Yeh(2016)]{credit}
I-Cheng Yeh.
\newblock {default of credit card clients}.
\newblock UCI Machine Learning Repository, 2016.
\newblock {DOI}: https://doi.org/10.24432/C55S3H.

\bibitem[Yeom \& Fredrikson(2021)Yeom and Fredrikson]{smoothing}
Samuel Yeom and Matt Fredrikson.
\newblock Individual fairness revisited: Transferring techniques from adversarial robustness.
\newblock In \emph{Proceedings of the Twenty-Ninth International Joint Conference on Artificial Intelligence}, IJCAI'20, 2021.
\newblock ISBN 9780999241165.

\bibitem[Yurochkin et~al.(2020)Yurochkin, Bower, and Sun]{Yurochkin2020Training}
Mikhail Yurochkin, Amanda Bower, and Yuekai Sun.
\newblock Training individually fair ml models with sensitive subspace robustness.
\newblock In \emph{International Conference on Learning Representations}, 2020.
\newblock URL \url{https://openreview.net/forum?id=B1gdkxHFDH}.

\bibitem[Zhang et~al.(2020)Zhang, Fang, Zhang, and Song]{ZKDT}
Jiaheng Zhang, Zhiyong Fang, Yupeng Zhang, and Dawn Song.
\newblock Zero knowledge proofs for decision tree predictions and accuracy.
\newblock In \emph{Proceedings of the 2020 ACM SIGSAC Conference on Computer and Communications Security}, CCS '20, pp.\  2039–2053, New York, NY, USA, 2020. Association for Computing Machinery.
\newblock ISBN 9781450370899.
\newblock \doi{10.1145/3372297.3417278}.
\newblock URL \url{https://doi.org/10.1145/3372297.3417278}.

\bibitem[Zhao et~al.(2021)Zhao, Wang, Wang, Li, Shen, and Feng]{VeriML}
Lingchen Zhao, Qian Wang, Cong Wang, Qi~Li, Chao Shen, and Bo~Feng.
\newblock Veriml: Enabling integrity assurances and fair payments for machine learning as a service.
\newblock \emph{IEEE Transactions on Parallel and Distributed Systems}, 32\penalty0 (10):\penalty0 2524--2540, 2021.
\newblock \doi{10.1109/TPDS.2021.3068195}.

\end{thebibliography}

\clearpage
\appendix
\clearpage
\section{Background Cntd.}\label{app:background}
\subsection{Polytopes}
The polytopes described succinctly by their linear inequalities (i.e., they are $H$-polytopes), which means that the number of halfspaces defining the polytope, denoted by $m$, is at most $O(poly(n))$, i.e. polynomial in the ambient dimension.

Next, we present a lemma which states that slicing a polyhedral complex with a hyperplane also results in a polyhedral complex.

\begin{lemma} Given an arbitrary polytope $\P := \{x | Ax \leq B\}$ and a hyperplane $H := \{x | c^Tx =
d\}$ that intersects the interior of $\P$, the two polytopes formed by the intersection of $\P$ and the each of
closed halfspaces defined by $H$ are polyhedral complices. \label{lemma:slice}
\end{lemma}

\begin{fact} Two ReLU activation codes of two neighboring polytopes differ in a single position and the differing bit corresponds to the facet common to both.
\label{fact:ReLU}
\end{fact}

\section{Individual Fairness Certification Cntd.}\label{app:IF}
\noindent\textbf{Algorithm.}
In this section, we describe the concrete algorithm to compute the local individiual fairness parameter for a data point $\x$ (Algorithm \ref{alg:IF}). 
Our construction is based on the Geocert algorithm by Jordan et. al (Algorithm \ref{alg:geo}, Section \ref{sec:background:PLNN}) for computing the pointwise $\ell_2$ robustness of neural networks with two key distinctions. First, we run on all the union of $(n-k)$-dimensional polytopes each of which corresponds to a fixed value of the sensitive feature set $\S$. Second, for each of these complices, we compute a lower bound on the pointwise $\ell_2$ robustness. The final certificate of fairness is the minimum over all the above bounds. 

In the following, we describe the working of the algorithm \ref{alg:IF} in more detail. 
 First, we compute the polyhedral complex $\PC$ for the model $\m$ (Step 1). Next for a fixed value of the set of the sensitive features $\S$ (Step 3), we compute the corresponding $(n-k)$-dimensional polyhedral complex $\PC'$ from the original $n$-dimensional polyhedral complex (\textsf{ReducePolyDim} function Alg. \ref{alg:app:n_kpoly}). The key idea is to fix the corresponding values of the features in $\S$ in the linear constraints of the polytopes in $\PC$. In the next step, we compute a lower bound on the pointwise $\ell_2$ robustness of $\x$ for the polyhedral complex $\PC'$ using the Geocert algorithm (Step 5-6). In particular, instead of minimizing the $\ell_2$ distance to a facet $\F$, we compute the projection of $\x$ onto a hyperplane $H$, where $\F$ lies entirely on $H$. The above computation is repeated for all the values of the set of sensitive features $\S$. The final certificate of fairness is the minimum of all the lower bounds as computed above (Step 8).

 In what follows, we briefly describe how to compute of the pointwise $\ell_2$ robustness of a point $x$. The problem essentially boils down to computing the largest $\ell_2$ ball centered at $x$ that fits within the union of n-dimensional polytopes defined by $f$. 
 
\begin{algorithm}[tbh]
   \caption{Geocert: Pointwise $\ell_2$ Robustness}
   \label{alg:geo}
\begin{algorithmic}[1]
\Statex \textbf{Input} $\x$ - Data point for pointwise $\ell_2$ robustness certification; $\m$ - Neural network; $dist$ - Distance Metric;
\Statex \textbf{Output} $\epsilon$ - Pointwise $\ell_2$ robustness certificate on $\x$; 
\State Compute all the polytopes for $\m$
\State Setup priority queue $Q\leftarrow [\thinspace ]$
\State Setup list of seen polytopes $C\leftarrow \{\mathcal{P}(x)\}$ \hfill \textcolor{blue}{$\rhd$} $\mathcal{P}(x)$ denotes the polytope containing $x$
\State \textbf{For} Facet $\mathcal{F} \in \mathcal{P}(x)$ \textbf{do}
\State $\hspace{0.5cm} Q.push (\textsf{ComputeDistance}(\mathcal{F},\x), \mathcal{F}, dist)$
\State \textbf{End For}
\State \textbf{While } $Q\neq \emptyset$ \textbf{ do} 
\State $\hspace{0.5cm}$ $(d,\mathcal{F})\leftarrow Q.pop()$
\State $\hspace{0.5cm}$ \textbf{If } $\textsf{IsBoundary}(\mathcal{F})==1$

\State $\hspace{1cm}$ \textbf{Return } $d$

\State $\hspace{0.5cm}$ \textbf{Else}
\State $\hspace{1cm}$ \textbf{For } $\mathcal{P} \in \N(\mathcal{F})\setminus C$ \textbf{do} 
\Statex \hfill \textcolor{blue}{$\rhd$} $\N(\mathcal{F})$ denote the two polytopes sharing the facet $\mathcal{F}$

\State $\hspace{1.5cm}$ \textbf{For } $\mathcal{F} \in \mathcal{P}$
\textbf{do}
\State $\hspace{2cm} Q.push (\textsf{ComputeDistance}(\mathcal{F},\x), \mathcal{F}, dist)$
\State $\hspace{1.5cm}$ \textbf{End For } 

\State $\hspace{1cm}$ \textbf{End For } 
\State $\hspace{0.5cm}$ \textbf{End If } 
\State \textbf{End While}

\end{algorithmic}
\end{algorithm}

\begin{algorithm}[H]
 \begin{algorithmic}[1]
 \caption{\scalebox{0.9}{\textsf{ReducePolyDim}} : Construct $(n-k)$-dimensional polytopes from $n$-dimensional polytopes}
   \label{alg:app:n_kpoly}
 \Statex\textbf{Inputs} \scalebox{0.9}{$\PC=\bigcup\P$} : Set of Polytopes where each polytope $\P$ is expressed as \scalebox{0.9}{$\{x|\mathbf{A}x\leq \mathbf{b}\}$}, $s = (s_1,\cdots,s_k) $ : Values of $k$ sensitive features
\Statex \textbf{Output} $\PC'$ : Set of $(n-k)$-dimensional Polytopes
\State $\PC' := \{\}$
\State \textbf{for} $\P \in \PC$
\State \hspace{0.5cm} \textbf{for} $i \in |row(\mathbf{A})|$
\State \hspace{1cm} \textbf{for} $j \in [k+1,n]$
\State \hspace{1.5cm}$\mathbf{A'}[i][j-k]=\mathbf{A'}[i][j]$
\State \hspace{1cm}\textbf{end for}
\State \hspace{1cm}$\mathbf{b'}[i]=\mathbf{b}[i]-\sum_{j=1}^k\mathbf{A}[i][j]\cdot s_j$
\State \hspace{0.5cm}\textbf{end for}
\State \hspace{0.5cm} Express $\P'=\{x|\mathbf{A'}x\leq \mathbf{b'}\}$
\State \hspace{0.5cm} $\PC':=\PC'\cup \P'$
\State \textbf{end for}
\State \textbf{Return} $\PC'$
\end{algorithmic}
\end{algorithm}

\clearpage
\begin{algorithm*}[tbh]
   \caption{\name: Verifiable Individual Fairness Certification}
   \label{alg:name}
\begin{algorithmic}[1]
\Statex \textbf{Input} $\x$ - Data point for fairness certification; $\W$ - Weights of the piecewise linear neural network; 
\Statex \textbf{Output}  $\epsilon$ - Local individual fairness parameter for $x$; $\CW$ - Commitment to the weights of the model; ZK proof that the $\epsilon$ is indeed a lower bound on $\epsilon_{IF}$ 
\Statex \textbf{Pre-Processing Offline Phase}
\State Construct the polyhedral complex \scalebox{0.9}{$\PC=\bigcup\P$} from $\textbf{W}$ where each polytope is expressed as \scalebox{0.9}{$\P=\{x|\mathbf{A}x\leq \mathbf{b}\}$}
\State Compute a reference point $z_i$ for each polytope $\P_i \in \PC$ such that $z_i \in \P_i$
\State Commit to the model weights $\CW$ and release them publicly 
\Statex \textbf{Online Phase}
\State $E=[\hspace{0.1cm}]$
\State \textbf{for} $(s_1,\cdots,s_k) \in domain(S_1)\times\cdots\times domain(S_k)$
\State \hspace{0.5cm} \textbf{for} $\P \in \PC$
\State \hspace{1cm} \textbf{for} $i \in |row(\mathbf{A})|$
\State \hspace{1.5cm} \textbf{for} $j \in [k+1,n]$
\State \hspace{2cm}$\mathbf{A'}[i][j-k]=\mathbf{A'}[i][j]$
\State \hspace{1.5cm}\textbf{end for}
\State \hspace{1.5cm}$\mathbf{b'}[i]=\mathbf{b}[i]-\sum_{j=1}^k\mathbf{A}[i][j]\cdot s_j$
\State\hspace{1cm}\textbf{end for}
\State \hspace{1cm} Express $\P'=\{x|\mathbf{A'}x\leq \mathbf{b'}\}$
\State\hspace{1cm} $\PC'=\PC'\cup \P'$
\State\hspace{0.5cm}\textbf{end for}
\State\hspace{0.3cm} $\big(\epsilon',\P_1, \langle(\F_1,d_1),\cdots,(\F_n,d_n)\rangle\big) = \G(\x,\PC',\dproj)$
\Statex \hfill\textcolor{blue}{$\rhd$} $\P_1$ is the first polytope traversed 
\Statex \hfill \textcolor{blue}{$\rhd$} \scalebox{0.9}{$\langle (\F_1,d_1),\cdots,(\F_n,d_n) \rangle $
} is the ordered sequence of the visited facets and their corresponding distances 

\vspace{0.2cm}
\State \hspace{0.3cm} Prover proves that $P_1$ is the polytope in $\PC'$ containing $\x$ \hfill \textcolor{blue}{$\rhd$} Using $\VPoly$
\State \hspace{0.3cm} Initialize the list of seen facets $T=[\hspace{0.1cm}]$
\State  \hspace{0.3cm} \textbf{for} facet $\F \in \N(\P_1)$
\State  \hspace{0.6cm} Prover proves that the computation of the distance $d$ from $\x$ to $\F$ is correct \textcolor{blue}{$\rhd$} Using $\VDist$
\State\hspace{0.6cm} $T.insert\big((\F,d)\big) ;$
\State \hspace{0.3cm} \textbf{end for}
\State \hspace{0.3cm} \textbf{for} $i \in [m-1]$
\State \hspace{0.6cm} Prover proves that $\F_i$ is indeed the facet with the smallest distance in $T$\hfill \textcolor{blue}{$\rhd$} Using $\VPQ$
\State \hspace{0.6cm} Prover proves that $\F$ is not a boundary facet \hfill \textcolor{blue}{$\rhd$} Using $\VBoundary$
\State \hspace{0.6cm} \textbf{for} $\P \in \N(\F_i)$
\State \hspace{0.9cm} Prover proves that $\P$ is a neighboring polytope sharing facet $\F$ \hfill \textcolor{blue}{$\rhd$} Using $\VNeighbor$
\State \hspace{0.9cm} \textbf{for} $\F \in \N(\P)$
\State \hspace{1.5cm} Prover proves that the computation of the distance $d$ from $\x$ to $\F$ is correct \hfill \textcolor{blue}{$\rhd$} Using $\VDist$
\State \hspace{1.5cm} $T.insert\big((\F,d)\big)$
\State \hspace{0.9cm} \textbf{end for}
\State \hspace{0.6cm} \textbf{end for}
\State \hspace{0.6cm} $T.remove\big((\F_i,d_i)\big)$

\State \hspace{0.3cm} \textbf{end for}
\State \hspace{0.3cm} Prover proves that $\F_m$ is indeed the facet with the smallest distance in $T_2$\hfill \textcolor{blue}{$\rhd$} Using $\VPQ$ 
\State \hspace{0.3cm} Prover proves that $\F_m$ is a boundary facet \hfill \textcolor{blue}{$\rhd$} Using $\VBoundary$ 
\State\hspace{0.3cm} $E.insert\big(d_m\big)$
\State \textbf{end for}
\State Prove that $\epsilon = \min E$ \hfill \textcolor{blue}{$\rhd$} Using $\VMin$ 

\end{algorithmic}
\end{algorithm*}

\begin{algorithm}[tbh]
   \caption{\VPoly}\label{alg:verify:polytope}
\begin{algorithmic}[1]
\Statex \textbf{Input} $\x$ - Data point for fairness certification; $\textsf{com}_{\mathbf{W}}$ - Committed weights of the piecewise linear neural network; $(s_1,\cdots,s_k)$ - Values of the sensitive features;
\Statex \textbf{Output} $\P'$ - Polytope corresponding to $\W$ containing $\x$;  $\R$ - ReLU activation code of $\x$; $\pi$ - ZK proof of the computation;
\vspace{0.2cm}
     \State Evaluate $\x$ on $\textsf{com}_{\mathbf{W}}$ to obtain ReLU activation code $\mathbf{R}$
    \State Compute the $n-k$-dimensional polytope  $\P=\{x|\A x\leq \mathbf{b}\}$ corresponding to $\mathbf{R}$ on $\textsf{com}_{\mathbf{W}}$ with $(s_1,\cdots,s_k)$ as the values of the sensitive features
    \State Generate proof $\pi$ of the above computation 
    \State \textbf{return} $(\P,\R,\pi)$
\end{algorithmic}
\end{algorithm}
\begin{algorithm}[tbh]
   \caption{\VDist}\label{alg:verify:distance}
\begin{algorithmic}[1]
\Statex \textbf{Input} $x^*$ - Data point for fairness certification; $\mathcal{F}$ - Facet; 

\Statex \textbf{Output} $d$ - Projected distance; $\pi$ - ZK proof of the computation;  
\vspace{0.2cm}
   \State Let $\mathcal{F} $ be represented as $a^T\cdot x= b$
   \State Compute $d=(\Big|b-a^Tx^*)/||a||\Big|$
    \State Generate proof $\pi$ of the above computation 
    \State \textbf{return} $(d,\pi)$
\end{algorithmic}
\end{algorithm}

\begin{algorithm}[tbh]
   \caption{\VNeighbor}\label{alg:verify:neighbor}
\begin{algorithmic}[1]
\Statex \textbf{Input} $\textsf{com}_{\mathbf{W}}$ - Weights of the piecewise linear neural network; $\F$ - Facet; $\P$ - Current polytope; $\R$ - ReLU activation code for $\P$; $z$ - Representative point for neighboring polytope; $(s_1,\cdots,s_k)$ - Values of the sensitive features;
\Statex \textbf{Output} $\P'$ - Neighboring polytope; $\RQ$ -  ReLU activation code of $\P'$; $\pi$ - ZK proof of the computation
\vspace{0.2cm}
\State $(\P',\RQ,\pi')\leftarrow \VPoly(z,\CW,(s_1,\cdots,s_k))$ \Statex \hfill\textcolor{blue}{$\rhd$}  Can be performed apriori in a pre-processing stage for efficiency
    \State \textbf{if} ($|\R-\RQ|_1\neq 1$) \hfill\textcolor{blue}{$\rhd$} Check hamming distance 1 between two binary vectors
    \State \hspace{0.3cm}\textbf{return} $\bot$ \State \textbf{if} $(\F \not \in \N(\P') \wedge (\F \not \in \N(\P)) )$ \hfill\textcolor{blue}{$\rhd$} Check facet $\F$ is common to both the polytopes 
\State \hspace{0.3cm}\textbf{return} $\bot$ 
\State Generate proof $\pi$ of the above computation
    \State \textbf{return} $(\P',\RQ,(\pi,\pi'))$
\end{algorithmic}
\end{algorithm}

\begin{algorithm}[tbh]
   \caption{\VBoundary}\label{alg:verify:boundary}
\begin{algorithmic}[1]
\Statex \textbf{Input}  $x^*$ - Data point for fairness certification; $\textsf{com}_{\mathbf{W}}$ - Weights of the piecewise linear neural network; $\F$ - Current facet represented as $\{x|\A x\leq \mathbf{b}\}$; $\P$ - Current polytope; $\R$ - ReLU activation code for $\P$; $z$ - Representative point for current facet $\F$ $(s_1,\cdots,s_k)$ - Values of the sensitive features;
\vspace{0.1cm} \Statex \textbf{Output} $b$ - Bit indicating boundary condition; $\pi$ - ZK proof of the computation 
\vspace{0.2cm}
    \State Compute the linear function $f_\R$ corresponding to activation code $\mathbf{R}$ on $\textsf{com}_{\mathbf{W}}$ with $(s_1,\cdots,s_k)$ as the values of the sensitive features
\Statex \hfill\textcolor{blue}{$\rhd$}  Can be performed apriori in a pre-processing stage for efficiency
\State \textbf{if} ($\A z > \mathbf{b}$)
\State \hspace{0.5cm} \textbf{return} $\bot$
\State \textbf{end if}
\State $b=1$
\State \textbf{for} $i \in [1,|\mathcal{Y}|-1]$
\State \hspace{0.5cm} $b\leftarrow b \cdot (f_\R(z)[0] == f_\R(z)[i])$ 
\Statex \hfill\textcolor{blue}{$\rhd$}  Testing that $f_\R(z)$  is equal on all of its elements 
\State \textbf{end for}
\State Generate proof $\pi$ of the above computation 
\State \textbf{return} $(b,(\pi,\pi'))$
\end{algorithmic}
\end{algorithm}

\begin{algorithm}[tbh]
   \caption{\VPQ}\label{alg:verify:PQ}
\begin{algorithmic}[1]
\Statex \textbf{Input}  $(\F,d)$ - Current facet with distance $d$; $\textbf{F}=\{(\F_1,d_1), \cdots,(\F_k,d_k)\}$ - List of all previously unseen facets and their distances;  
\Statex \textbf{Output} $\pi$ - ZK proof of the computation 
\vspace{0.2cm}
\State \textbf{for} $\F_i \in \textbf{F}$
\State \hspace{0.5cm} \textbf{if} $(d>d_i)$
\State\hspace{1cm} \textbf{return} $\bot$
\State\hspace{0.5cm} \textbf{end if}
\State \textbf{end for}
\State Generate proof $\pi$ of the above computation 
\State\textbf{return} $\pi$
\end{algorithmic}
\end{algorithm}
\begin{algorithm}[tbh]
   \caption{\VMin}\label{alg:verify:min}
\begin{algorithmic}[1]
\Statex \textbf{Input}  $E$ - List of values; $\epsilon^*$ - Individual fairness parameter;
\Statex \textbf{Output} $\pi$ - ZK proof of the computation 
\vspace{0.2cm}
\State \textbf{for} $\epsilon \in E$
\State \hspace{0.5cm} \textbf{if} $(\epsilon^* > \epsilon)$
\State\hspace{1cm} \textbf{return} $\bot$
\State\hspace{0.5cm} \textbf{end if}
\State \textbf{end for}
\State Generate proof $\pi$ of the above computation 
\State\textbf{return} $\pi$
\end{algorithmic}
\end{algorithm}
\begin{algorithm}[tbh]
   \caption{\VInf}\label{alg:verify:inference}
\begin{algorithmic}[1]
\Statex \textbf{Input} $\x$ - Data point for fairness certification; $\textsf{com}_{\mathbf{W}}$ - Committed weights of the piecewise linear neural network $\m$; 
\Statex \textbf{Output} $y$ - The prediction $\m(\x)$; $\pi$ - ZK proof of the computation;
\vspace{0.2cm}
     \State Evaluate $\x$ on $\textsf{com}_{\mathbf{W}}$ to obtain ReLU activation code $\mathbf{R}$
     \State Compute the linear function $f_\R$ corresponding to activation code $\mathbf{R}$ on $\textsf{com}_{\mathbf{W}}$
     \State Compute $f_\R(\x)$ 
     \State $y=\arg \max_{i \in [|\mathcal{Y}|]}f_\R(\x)$
\State Generate proof $\pi$ of the above computation 
\State\textbf{return} ($y, \pi$)
\end{algorithmic}
\end{algorithm}
\clearpage
\section{Correctness of \name}\label{app:sec:correctness_fairproof_proj}
In this section, we prove the correctness of \name given in Alg. \ref{alg:name}.
First, we re-state the correctness of \G.
\begin{theorem}[Correctness of \G~\cite{Geocert}]
For a fixed polyhedral complex $\PC$, a fixed point $\x$ and a distance function $\phi$ that satisfies ray monotonocity, \G~returns a boundary facet with the minimum distance. \label{thm:correctness:Geo}
\end{theorem}
\begin{fact}The projection of a given point $\x$ onto a hyperplane $H$ where $\F \subseteq H$ gives a lower bound on its $\ell_2$ distance to $\F$, i.e., $\dproj(x,\F)\leq \dlp(x,\F)$. \label{fact:proj}
\end{fact}
\begin{theorem} Let $\m$ be a piecewise-linear neural network. Replacing in Algorithm \ref{alg:geo} with $\dlp(\cdot)$ distance with $\dproj(\cdot)$ gives a lower bound on the individual fairness guarantee, i.e., $\epsilon_{\dproj} \leq \epsilon_{\dlp}$.\label{thm:projection_distance}
\end{theorem}
\begin{proof}
We will prove by contradiction. Let $\PC$ be the polyhedral complex associated with the model $\m$. 
Let us assume that there exists a boundary facet $\F^*$ such that $\dlp(\F,x)<\epsilon_{\dproj}$. Now if the corresponding polytope $\P_{\F^*}$ was traversed by $\G(x,\PC,\dproj)$, then all the facets in $\P_{\F^*}$ including $\F^*$ were checked. Then from the correctness of $\G$ (Thm. \ref{thm:correctness:Geo}), this leads to a contradiction of \ref{fact:proj}. Now let us consider the alternative case where $\P_{\F^*}$ was not traversed by $\G(x,\PC,\dproj)$.  From Thm. \ref{thm:correctness:Geo} this means that there exists another boundary facet $\F^*$ such that $\dproj(x,\F^*)\leq \dproj(x,\F)$. Then by Fact \ref{fact:proj}, $\dproj(\F^*,x)=\epsilon_{\dproj}\leq \dproj(\F,x)\leq \dlp(\F,x)$ which contradicts our assumption.  
\end{proof}
\begin{theorem}[Correctness of \name] 
For a given data point $\x$, \name~(Algorithm \ref{alg:name}) generates $\epsilon$ such that $\epsilon \leq \epsilon_{IF}$. \label{thm:correctness:name}
\end{theorem}
\begin{proof}
The proof of the above theorem follows directly from Theorem \ref{thm:correctness:Geo},  Theorem \ref{thm:projection_distance} and Fact \ref{fact:proj}.
\end{proof}
\section{Security Proof} \label{app:proof}
\begin{compactenum}
\item  \textbf{Completeness} 
\begin{gather}
\forall x, \mathbf{W}\\
\mathrm{Pr}\begin{bmatrix*}[l] \textsf{pp}\leftarrow \name.\textsf{KeyGen}(1^\lambda)\\\textsf{com}_{\mathbf{W}}\leftarrow \name.\textsf{Commit}(\mathbf{W},\textsf{pp},r)\\
(y,\epsilon,\pi)\leftarrow\name.\textsf{Prove}(\mathbf{W},x,\textsf{pp},r)
\\
\name.\textsf{Verify}(\textsf{com}_{\mathbf{W}},x,y,\epsilon,\pi,\textsf{pp})=1\end{bmatrix*} =1
    \end{gather}
\item
\textbf{Soundness}
\begin{gather}\mathrm{Pr} \begin{bmatrix*}[l] \textsf{pp}\leftarrow \name.\textsf{KeyGen}(1^\lambda)\\ (\mathbf{W^*},\textsf{com}_{\mathbf{W^*}},\textbf{X},\epsilon^*,y^*,\pi^*,r )\leftarrow \calA(1^\lambda,\textsf{pp})\\\textsf{com}_{\mathbf{W^*}}\leftarrow \name.\textsf{Commit}(\mathbf{W^*},r))\\\name.\textsf{Verify}(\textsf{com}_{\mathbf{W^*}},x,y^*,\epsilon^*,\pi^*,\textsf{pp})=1\\\big(\exists \tilde{x}, d(x,\tilde{x})\leq\epsilon \wedge \m(\mathbf{W^*},\mathbf{X})\neq \m(\mathbf{W^*},\mathbf{\tilde{X}})\big)\\ \hspace{3.8cm} \vee y\neq \m(\mathbf{W^*},\mathbf{X})  \end{bmatrix*} <negl(\lambda)\end{gather}
\item \textbf{Zero-Knowledge} Let $\lambda$ be the security parameter obtained from $\lambda, \textsf{pp}\leftarrow \name.KeyGen(1^\lambda)$
\begin{multline}
|\mathrm{Pr}[\textsf{Real}_{\calA,\mathbf{W}}(\textsf{pp})=1]-\mathrm{Pr}[\textsf{Ideal}_{\calA,\calS^{\calA}}(\textsf{pp})=1]|\\\leq \textsf{negl}(\lambda)
\end{multline}
\end{compactenum}
\begin{figure}\centering\fbox{\begin{varwidth}{0.7\columnwidth}$\textsf{Real}_{\calA,\mathbf{W}}(\textsf{pp}):$ 
\begin{compactenum}
\item $\textsf{com}_\textbf{W}\leftarrow \name.\textsf{Commit}(\mathbf{W},\textsf{pp},r)$ \item $x\leftarrow \calA(\textsf{com}_{\mathbf{W}},\textsf{pp})$\item $(y,\epsilon,\pi)\leftarrow\name.\textsf{Prove}(\mathbf{W},x,\textsf{pp},r)$
\item $b\leftarrow\calA(\textsf{com}_{\mathbf{W}},x,y,\epsilon,\pi,\textsf{pp}) $ \item Output $b$
\end{compactenum}
\end{varwidth}}
\end{figure}
\begin{figure}
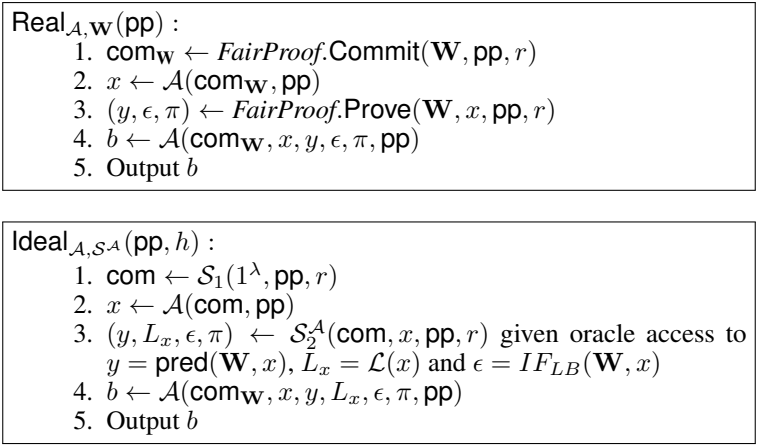
\centering\fbox{\begin{varwidth}{0.7\columnwidth}$\textsf{Ideal}_{\calA,\calS^\calA}(\textsf{pp},h):$ 
\begin{compactenum}
\item $\textsf{com}\leftarrow\calS_1(1^\lambda,\textsf{pp},r)$
\item $x\leftarrow \calA(\textsf{com},\textsf{pp})$
\item $(y,L_x,\epsilon,\pi)\leftarrow\calS_2^{\calA}(\textsf{com},x,\textsf{pp},r)$ given oracle access to $y=\textsf{pred}(\mathbf{W},x)$, $L_x=\mathcal{L}(x)$ and $\epsilon=IF_{LB}(\mathbf{W},x)$
\item $b\leftarrow\calA(\textsf{com}_{\mathbf{W}},x,y,L_x,\epsilon,\pi,\textsf{pp}) $ \item Output $b$
\end{compactenum}
\end{varwidth}}
\caption{Zero-knowledge games}
\end{figure}
\begin{proofs}
\par
\textbf{Completeness.} The completeness guarantee follows trivially from our construction. \\\\
\noindent\textbf{Soundness.} 
$\mathcal{L}(x)$ denotes the leakage function for \name, specifically, $\mathcal{L}(x)=\{n_1,\cdots,n_{|\S|}\}$, where $n_i$ denotes the number of facets traversed for the $i$-th value of the sensitive attribute $\S$.

Recall, the functioning of \G~can be summarized as follows:
\begin{compactenum}
\item Start traversing from the polytope containing $\x$.
\item Compute the distances to all the facets of the current polytope and store them.
\item Select the hitherto unseen facet with the smallest distance.
\item Stop if this is a boundary facet. 
\item Else, traverse next to the neighboring polytope that shares the current facet.
\end{compactenum}
A malicious prover can cheat in any (or a combination) of the above steps.
We will consider each of them separately as follows. 
\begin{lemma}[Soundness of \VPoly] Let $\P=\{x|\mathbf{A}x\leq \mathbf{B}\}$ be the correct polytope obtained from the piecewise-linear neural network with weights $\W$ for a given value of the sensitive features. For any polytope $\P'=\{\mathbf{A}'x<\mathbf{b}'\}$ such that $(\mathbf{A}\neq \mathbf{A}') \lor (\mathbf{b}\neq \mathbf{b'})$, we have
  \end{lemma}
 \begin{gather} \mathrm{Pr}[\name.\textsf{Verify}(\textsf{com}_{\mathbf{W^*}},x,y^*,\epsilon^*,\pi^*,\textsf{pp})=1]< \textsf{negl}(\lambda)
    \end{gather}
\begin{proofs} As shown in Alg. \ref{alg:verify:polytope}, the verification process re-computes the correct polytope from the committed model weights. The only way the prover can cheat is if they can produce a $P'$ such that $Open(\textsf{com}_\P)=\P'$ which violates the binding property of the commitment scheme. 
\end{proofs}
\begin{lemma} [Soundness of \VDist] For a given facet $\F=\{Ax\leq b\}$, data point $\x$, and value $d'$ such that $d'\neq \big |\frac{b-A^T\x}{\lVert A\rVert}\big |$, we have: \end{lemma} 
\begin{gather} \mathrm{Pr}[\name.\textsf{Verify}(\textsf{com}_{\mathbf{W^*}},x,y^*,\epsilon^*,\pi^*,\textsf{pp})=1]< \textsf{negl}(\lambda)
    \end{gather}

\begin{proofs} The verification process (Alg. \ref{alg:verify:distance}) re-computes the correct distance. Hence, the only way the prover can cheat is if they can produce a $d'$ such that $Open(\textsf{com}_d)=d'$ which violates the binding property of the commitment scheme. 
\end{proofs}

\begin{lemma}[Soundness of \VPQ]
Let $\textbf{d}=\{d_1,\cdots,d_k\}$ be a set of values such that $d_{min}=\min_i d_i$. For any value $d'$ such that $d' > d_{min}$, we have:
\end{lemma}
\begin{gather}
\mathrm{Pr}[\name.\textsf{Verify}(\textsf{com}_{\mathbf{W}},x,y^*,\epsilon^*,\pi^*,\textsf{pp})=1]< \textsf{negl}(\lambda)
\end{gather}
\begin{proofs} The verification checks the minimality of the given value against all values in $\textbf{d}$ (Alg. \ref{alg:verify:PQ}). The only way to cheat would require producing a $\textbf{d}$ with a different minimum which violates the binding property of the commitment scheme.
\end{proofs}

\begin{lemma}[Soundness of \VBoundary]
Consider a piecewise-linear neural network with weights $\mathbf{W}$. 
For any facet $\F$ such that which is not a boundary facet, we have
\end{lemma}
\begin{gather}
\mathrm{Pr}[\name.\textsf{Verify}(\textsf{com}_{\mathbf{W}},x,y^*,\epsilon^*,\pi^*,\textsf{pp})=1]\leq \textsf{negl}(\lambda)
\end{gather}
\begin{proof} The verification algorithm computes the linear function corresponding to the given activation code (Alg. \ref{alg:verify:boundary}. A prover can cheat here only if they can compute a different linear function $f'$ which would require violating the binding property of the commitment scheme.
\end{proof}

\begin{lemma}[Soundness of \VNeighbor] Let $\P=\{x|\mathbf{A}x\leq \mathbf{b}\}$ be a polytope belonging to the polyhedral complex of the piecewise-linear neural network with weights $\mathbf{W}$ and let $\F \in \N(\P)$. Let $\bar{\P}=\{x|\mathbf{\bar{A}}x\leq \mathbf{\bar{b}}\}$ and $\P$ be neighboring polytopes, sharing the facet $\F$, i.e.,  $\bar{\P} \in \N(\F)\setminus \P$. Let $z \in \R^n$  be a data point. For any polytope $P'=\{x|\mathbf{A'}x\leq \mathbf{b'}\}$ such that $(\mathbf{\bar{A}}\neq \mathbf{A'}) \wedge (\mathbf{\bar{b}}\neq \mathbf{b'})$, we have
\end{lemma}
\begin{gather}
\mathrm{Pr}[\name.\textsf{Verify}(\textsf{com}_{\mathbf{W}},x,y^*,\epsilon^*,\pi^*,\textsf{pp})=1]< \textsf{negl}(\lambda)
\end{gather}

\begin{proofs}  The verification algorithm first checks whether $\bar{P}$ contains the reference point $z$ (Alg. \ref{alg:verify:neighbor}). The soundness of this follows from $\VPoly$. Cheating on the next steps (checking the hamming distance and facet intersection) means that the prover is essentially able to generate a polytope $P'$ such that $Open(\textsf{com}_\P)=\P'$ which violates the binding property of the commitment scheme. 
\end{proofs}

\textbf{Zero-Knowledge.} The zero-knowledge property follows directly
from the commitment scheme and the zero-knowledge backend proof system we use. 
We note that the
zero-knowledge proof protocol itself is not the focus of this paper; instead, we show how we can use
existing zero-knowledge proof protocols to provide verifiable individual fairness certification in a smart way for high
efficiency. 
\end{proofs}

\section{Evaluation Cntd.} \label{app:evaluation}

\begin{table}[ht]
  \centering
  \smallskip
  \scalebox{1}{
  \begin{tabular}{l|c|c|c|c}

    Dataset-Model & Online  (in mins) & Offline  (in mins) & Improvement & Traversals\\
    \toprule
    
    German (4,2) &4.90 $\pm$ 0.12& 3.61 $\pm$ 0.19 & 1.74$\times$ & 40 $\pm$ 3 
    \\
     \hline

    German (2,4) & 1.17 $\pm$ 0.02& 0.67 $\pm$ 0.03 & 1.57 $\times$ & 13$\pm$ 1
    \\

    \midrule

    Credit (4,2) &  3.52 $\pm$ 0.08 & 2.31 $\pm$ 0.10 & 1.66$\times$ & 28 $\pm$ 2
    \\

    \hline

    Credit (2,4) & 2.08 $\pm$ 0.04 & 1.11 $\pm$0.07 & 1.49 $\times$ & 25 $\pm$ 1 
    \\

    \midrule

    Adult (4,2) &  3.94 $\pm$0.10 & 1.72 $\pm$ 0.08 & 1.43 $\times$ & 41 $\pm$ 3 
    \\

    \hline
    Adult (8,2)  &  3.94 $\pm$ 0.30 & 1.34 $\pm$ 0.08 & 1.36 $\times$ & 38 $\pm$ 8 
    \\
    \bottomrule
    
    \end{tabular}}
      \caption{\label{tab:proof} Time for proof generation averaged over 100 randomly sampled data points. Mean and standard error are reported for each dataset-model. Offline computations are done in the initial setup phase of \name while Online computations are done for every new query. Improvement = (Online time + Offline time)/ Online time. Traversals gives the total number of iterations (also total number of popped facets) of \textsf{GeoCert} ran by \name.}
\end{table}

\begin{figure*}[hbt!]
\center
    \begin{subfigure}[b]{0.333\linewidth}
        \centering
         \includegraphics[width=\linewidth]{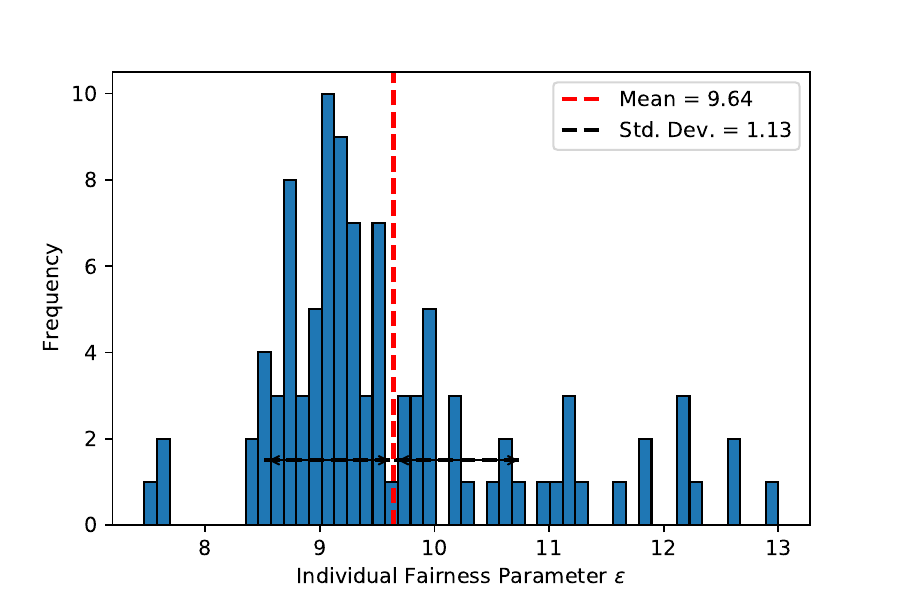}
            \caption{\textit{Credit, wd=2.5}}
        \label{fig:credit:fair42}
    \end{subfigure}
    \begin{subfigure}[b]{0.333\linewidth}
    \centering \includegraphics[width=\linewidth]{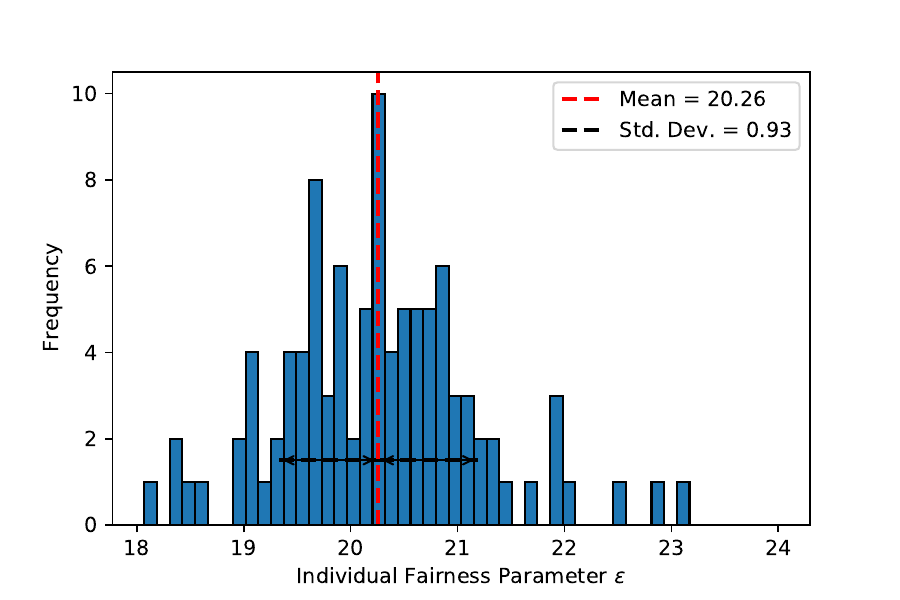}   
\caption{\textit{Adult, wd=0.2} }
        \label{fig:adult:fair42}\end{subfigure}
    \begin{subfigure}[b]{0.333\linewidth}
    \centering    \includegraphics[width=\linewidth]{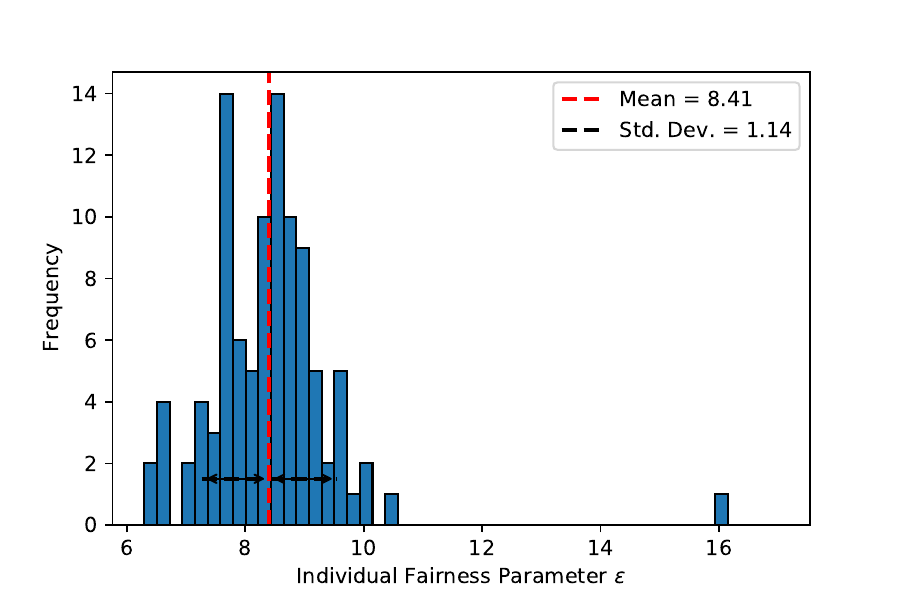}
    \caption{\textit{German, wd=10} }
        \label{fig:german:fair42}\end{subfigure}
        \caption{Histogram of fairness parameter $\epsilon$ for fair models of size (4,2). `wd' represents the values of the Weight decay parameter.}
        \label{fig:fair42}
        \end{figure*}

\begin{figure*}[hbt!]
\center
    \begin{subfigure}[b]{0.333\linewidth}
        \centering
         \includegraphics[width=\linewidth]{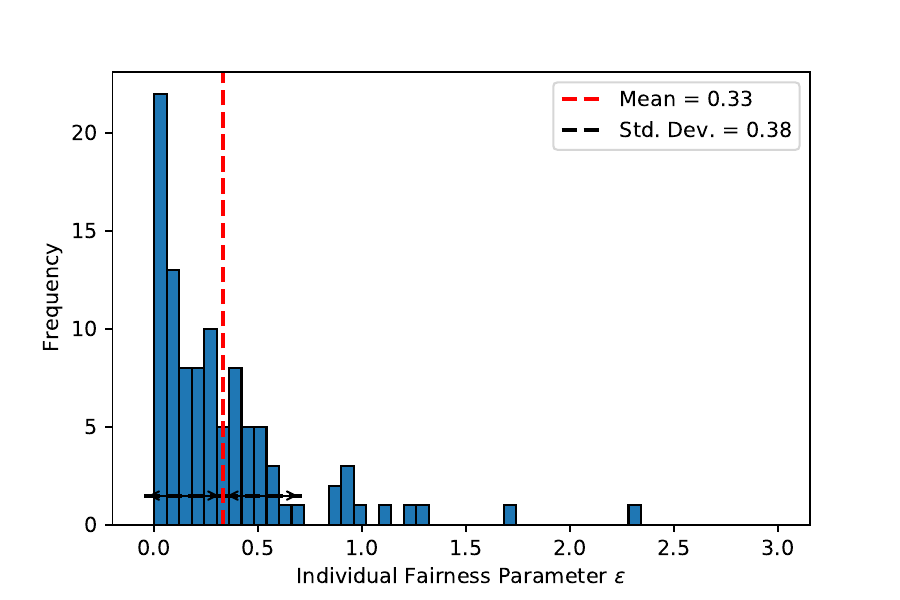}
            \caption{\textit{Credit}}
        \label{fig:credit:unfair42}
    \end{subfigure}
    \begin{subfigure}[b]{0.333\linewidth}
    \centering \includegraphics[width=\linewidth]{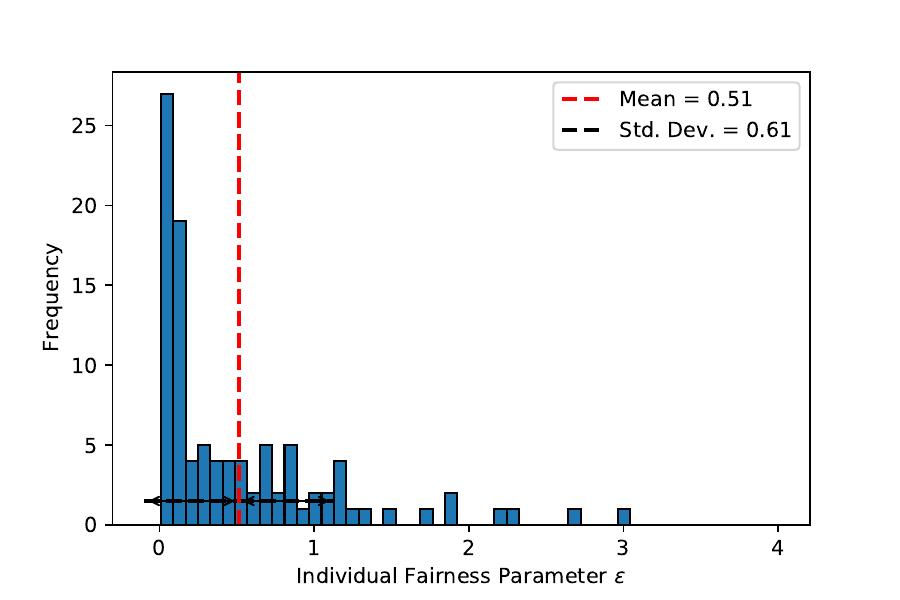}   
\caption{\textit{Adult} }
        \label{fig:adult:unfair42}\end{subfigure}
    \begin{subfigure}[b]{0.333\linewidth}
    \centering    \includegraphics[width=\linewidth]{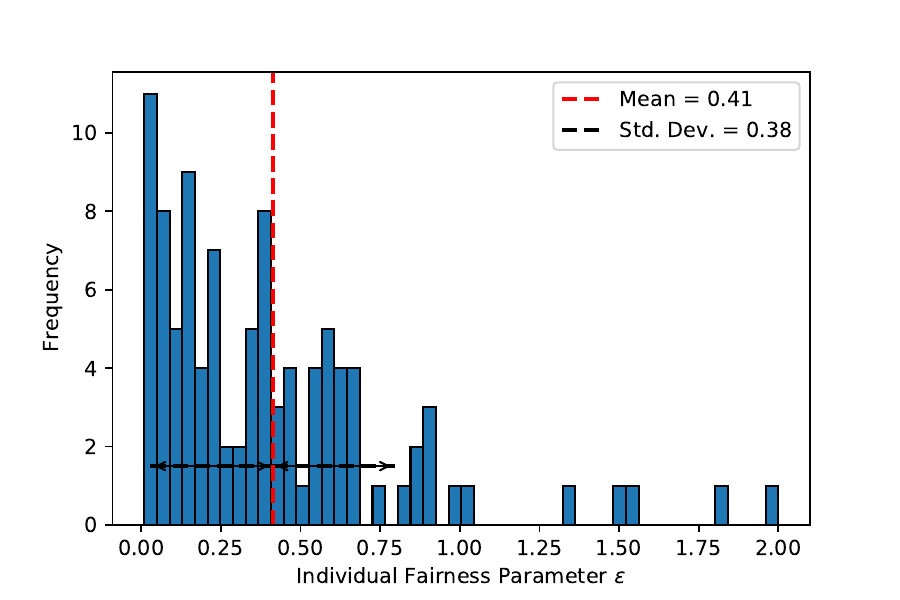}
    \caption{\textit{German} }
        \label{fig:german:unfair42}\end{subfigure}
        \caption{Histogram of fairness parameter $\epsilon$ for unfair models of size (4,2). Weight decay is set to zero here for all.}
        \label{fig:unfair42}
        \end{figure*}

\begin{figure*}[hbt!]
\center
    \begin{subfigure}[b]{0.333\linewidth}
        \centering
         \includegraphics[width=\linewidth]{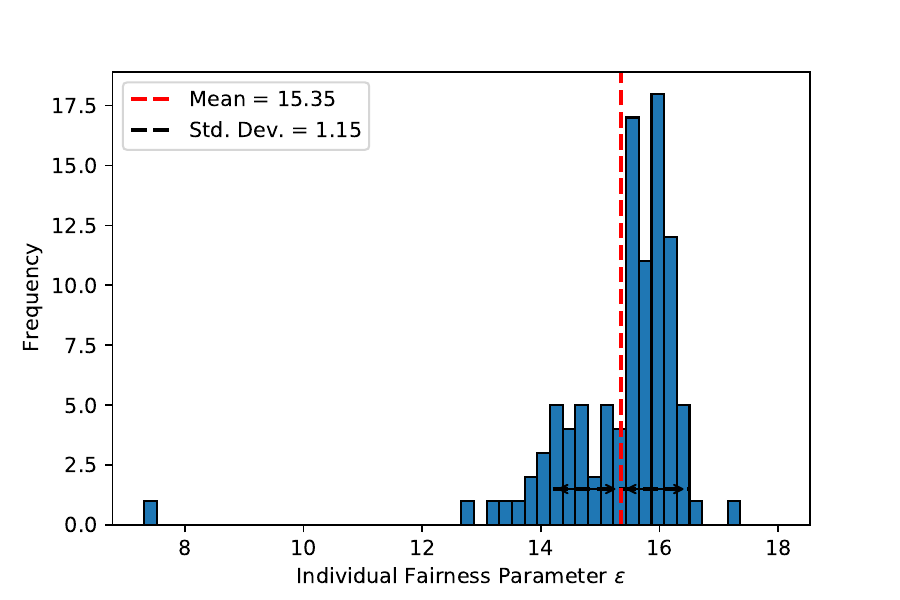}
            \caption{\textit{Credit, wd=3}}
        \label{fig:credit:fair82}
    \end{subfigure}
    \begin{subfigure}[b]{0.333\linewidth}
    \centering \includegraphics[width=\linewidth]{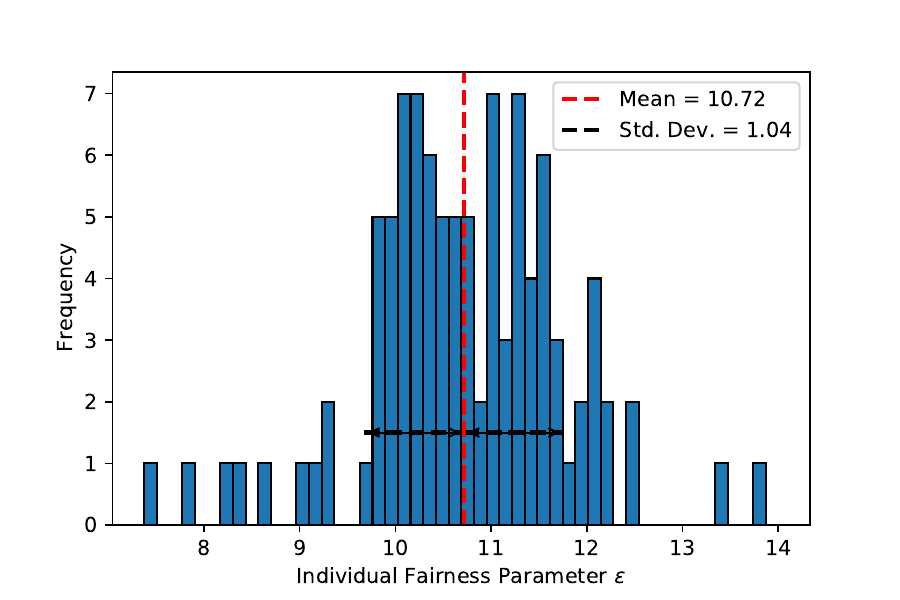}   
\caption{\textit{Adult, wd=0.2} }
        \label{fig:adult:fair82}\end{subfigure}
    \begin{subfigure}[b]{0.333\linewidth}
    \centering    \includegraphics[width=\linewidth]{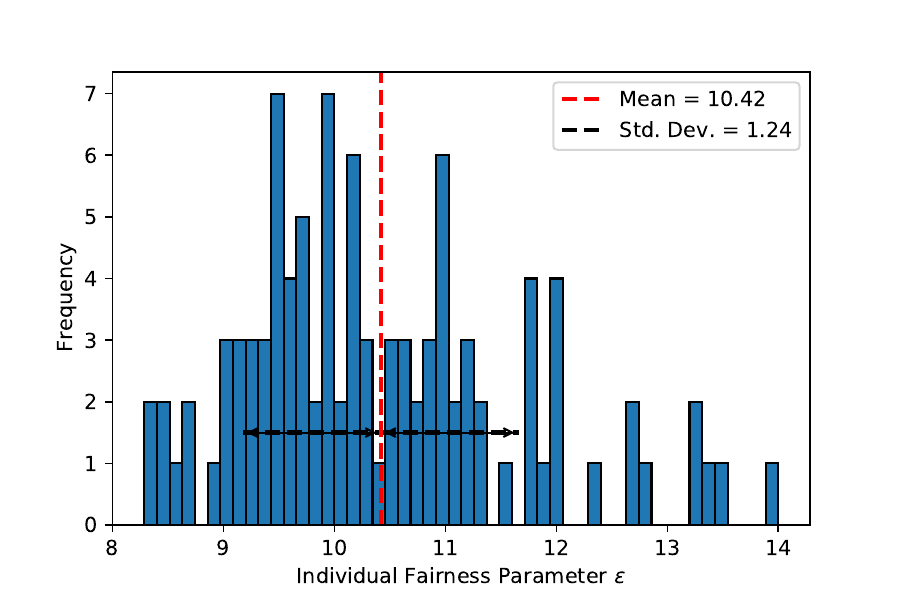}
    \caption{\textit{German, wd=10} }
        \label{fig:german:fair82}\end{subfigure}
        \caption{Histogram of fairness parameter $\epsilon$ for fair models of size (8,2). `wd' represents the values of the Weight decay parameter.}
        \label{fig:fair82}
        \end{figure*}

\begin{figure*}[hbt!]
\center
    \begin{subfigure}[b]{0.333\linewidth}
        \centering
         \includegraphics[width=\linewidth]{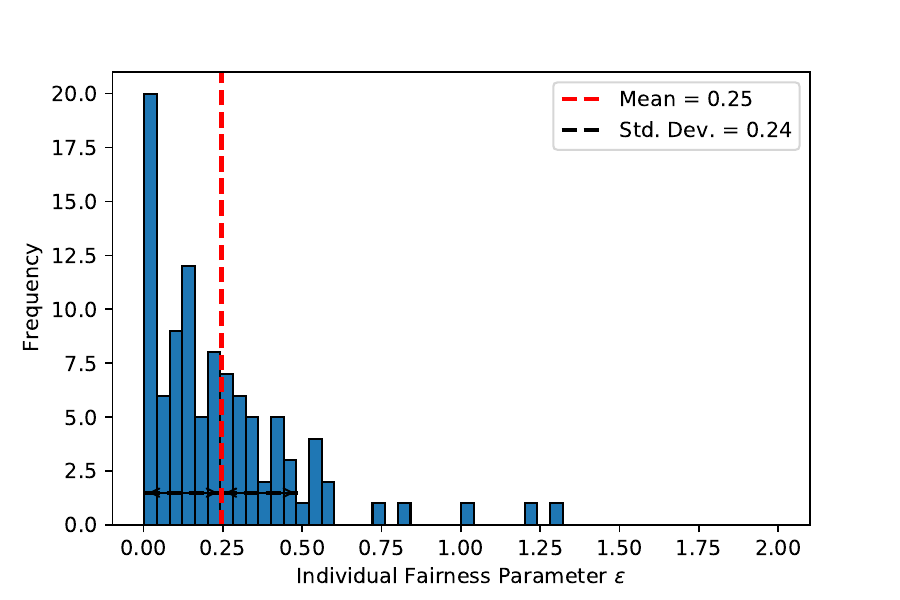}
            \caption{\textit{Credit}}
        \label{fig:credit:unfair82}
    \end{subfigure}
    \begin{subfigure}[b]{0.333\linewidth}
    \centering \includegraphics[width=\linewidth]{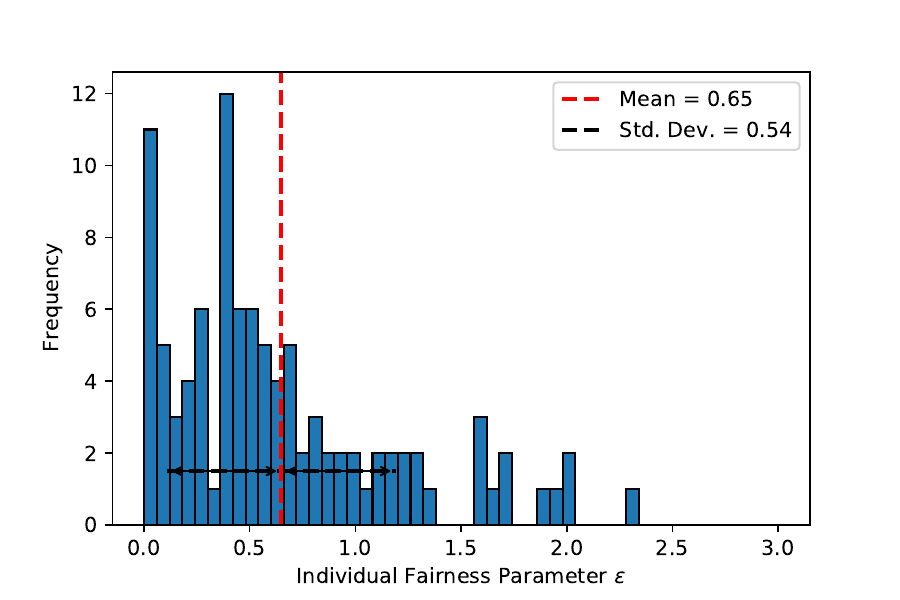}   
\caption{\textit{Adult} }
        \label{fig:adult:unfair82}\end{subfigure}
    \begin{subfigure}[b]{0.333\linewidth}
    \centering    \includegraphics[width=\linewidth]{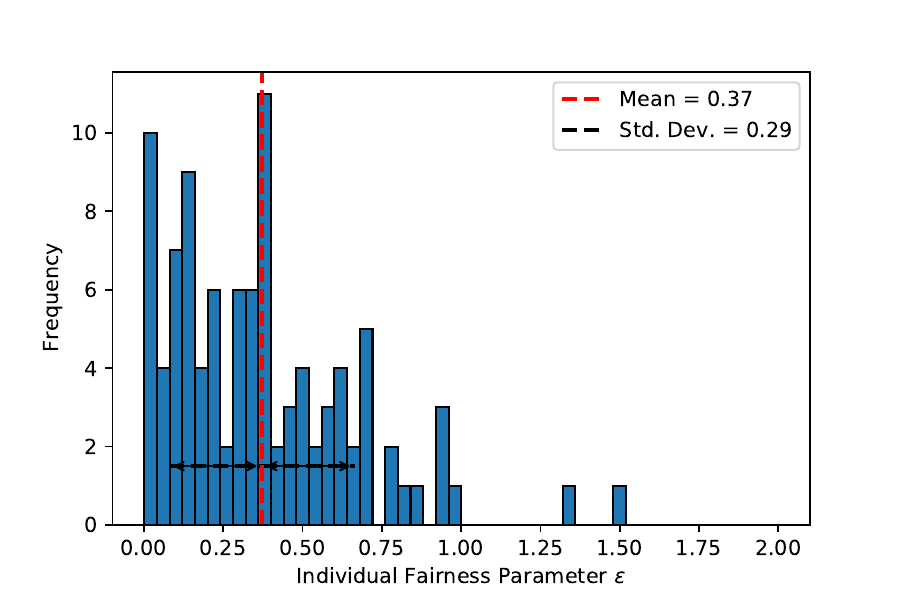}
    \caption{\textit{German} }
        \label{fig:german:unfair82}\end{subfigure}
        \caption{Histogram of fairness parameter $\epsilon$ for unfair models of size (8,2). Weight decay is set to zero here for all.}
        \label{fig:unfair82}
        \end{figure*}

\begin{figure*}[bht!] 
\center
\begin{subfigure}[b]{0.32\linewidth}
\centering \includegraphics[width=\linewidth]{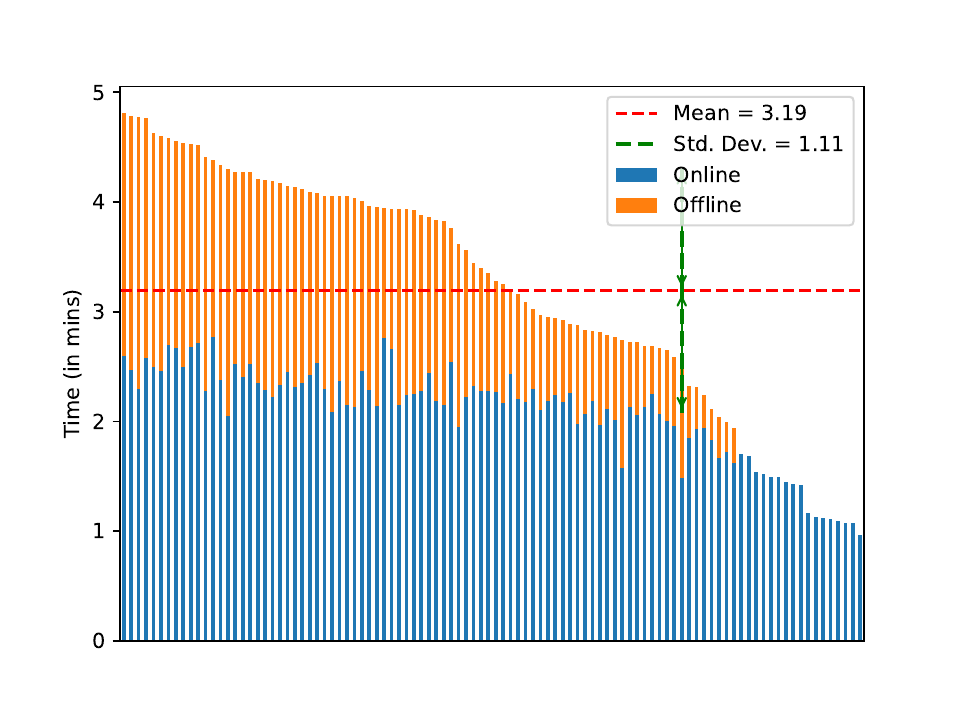}
\caption{\textit{Credit (2,4)}}
\label{fig:credit:online_offline24}\end{subfigure}
\begin{subfigure}[b]{0.32\linewidth}
\centering \includegraphics[width=\linewidth]{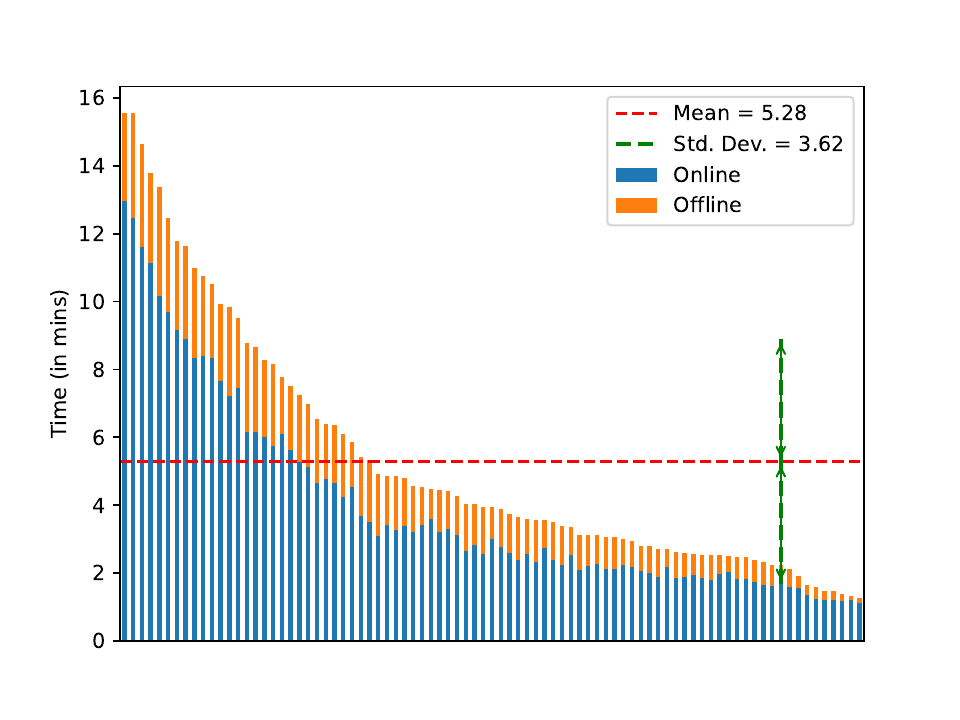}   
\caption{\textit{Adult (8,2)}}
\label{fig:adult:online_offline82}\end{subfigure}
\begin{subfigure}[b]{0.36\linewidth}
\centering   \includegraphics[width=\linewidth]{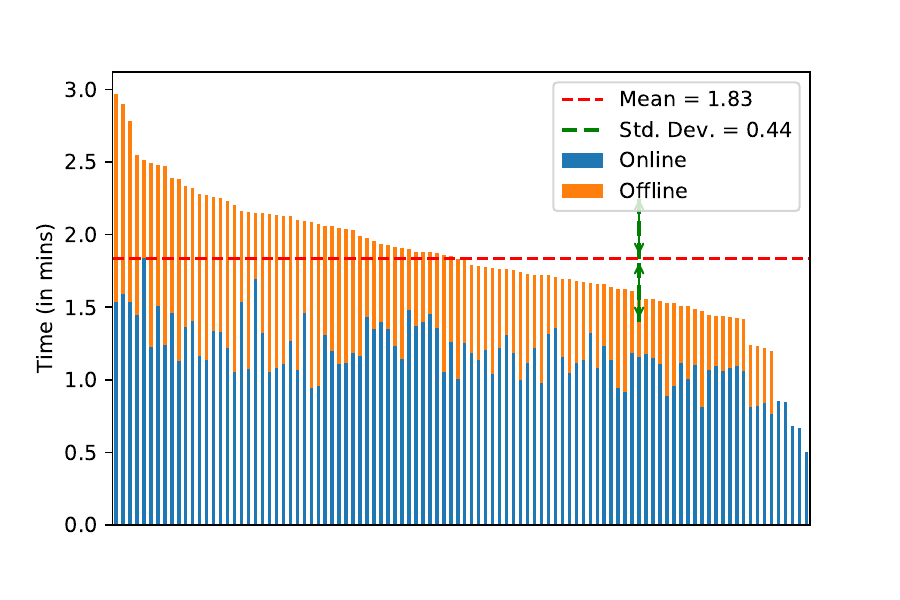}
\caption{\textit{German (2,4)}}
\label{fig:german:online_offline24}\end{subfigure}\\
\begin{subfigure}[b]{0.345\linewidth}
\centering\includegraphics[width=\linewidth]{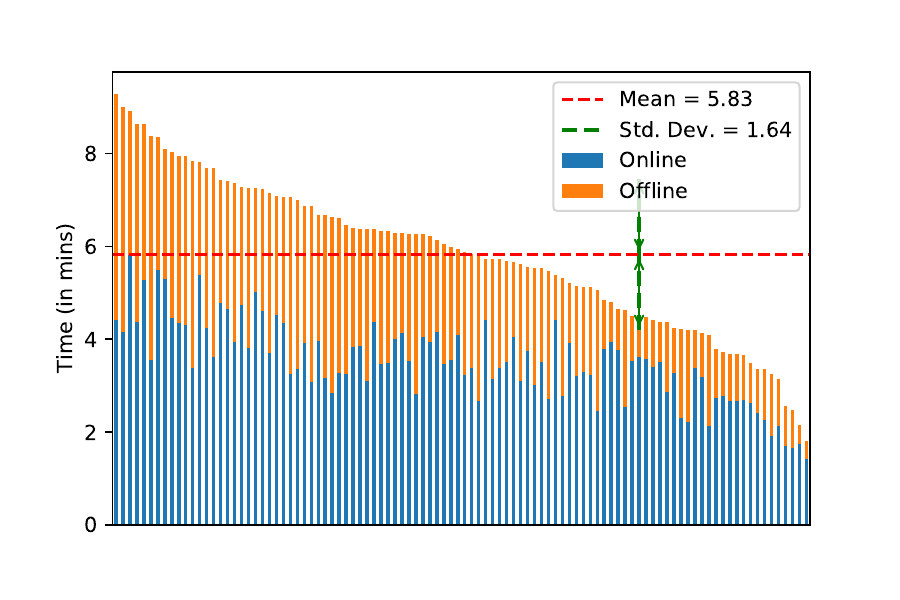}
    \caption{\textit{Credit (4,2)}}
\label{fig:credit:online_offline42}\end{subfigure}
\begin{subfigure}[b]{0.31\linewidth}
\centering \includegraphics[width=\linewidth]{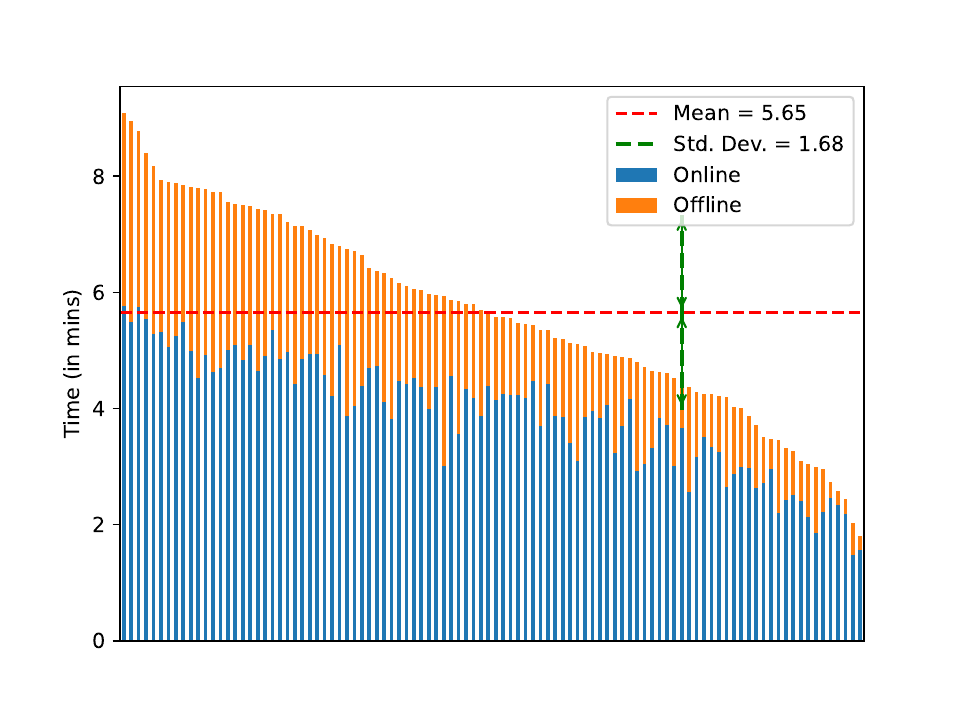}   
\caption{\textit{Adult (4,2)}}
\label{fig:adult:online_offline42}\end{subfigure}
\begin{subfigure}[b]{0.345\linewidth}
\centering    \includegraphics[width=\linewidth]{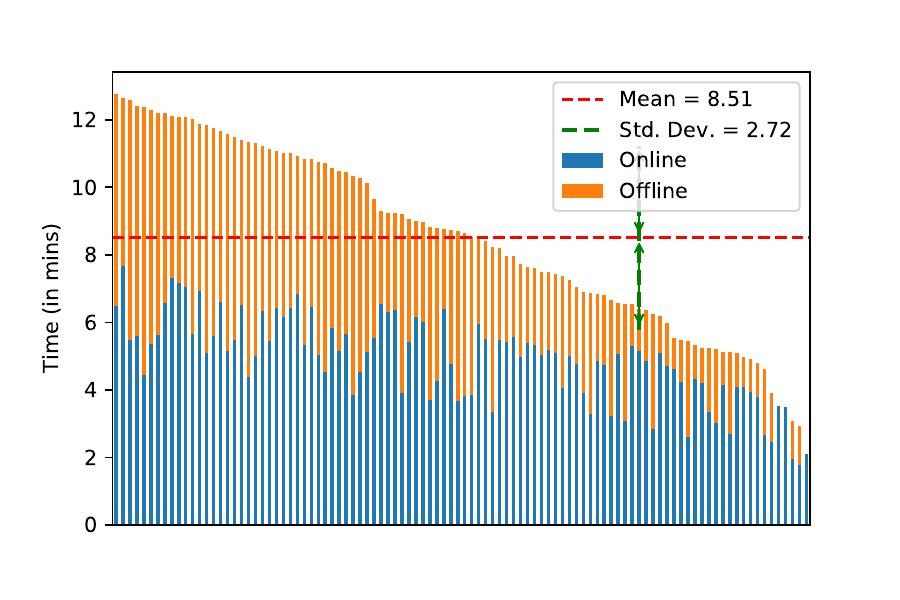}
\caption{\textit{German (4,2)}}
\label{fig:german:online_offline42}\end{subfigure}

 \caption{Proof generation time for 100 random data points.} 
 \label{fig:proof}
\end{figure*}

\begin{figure*}[hbt!] 
\center
\begin{subfigure}[b]{0.32\linewidth}
    \centering
     \includegraphics[width=\linewidth]{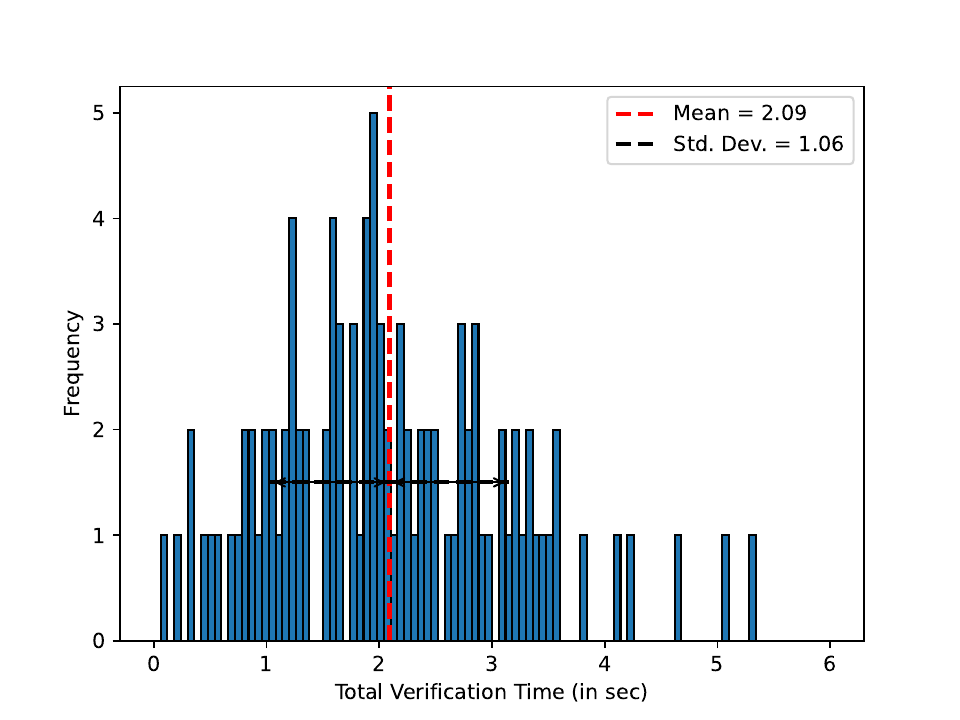}
        \caption{\textit{Credit (2,4)}}
    \label{fig:credit:verify}
\end{subfigure}
\begin{subfigure}[b]{0.32\linewidth}
\centering \includegraphics[width=\linewidth]{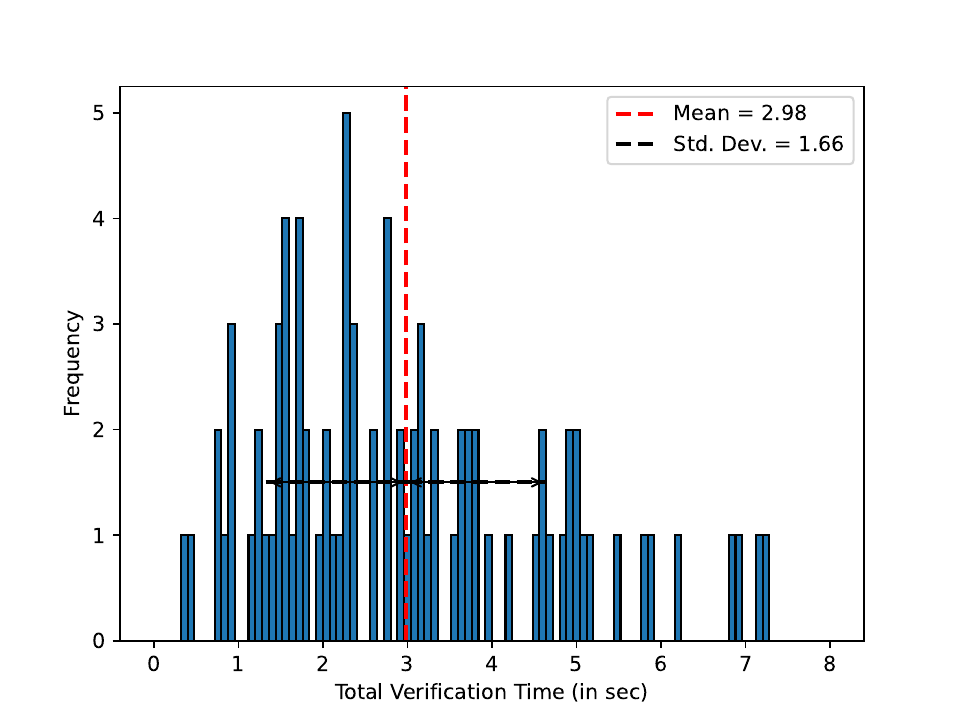}   
\caption{\textit{Adult (8,2)}}
    \label{fig:adult:verify}\end{subfigure}
\begin{subfigure}[b]{0.36\linewidth}
\centering    \includegraphics[width=\linewidth]{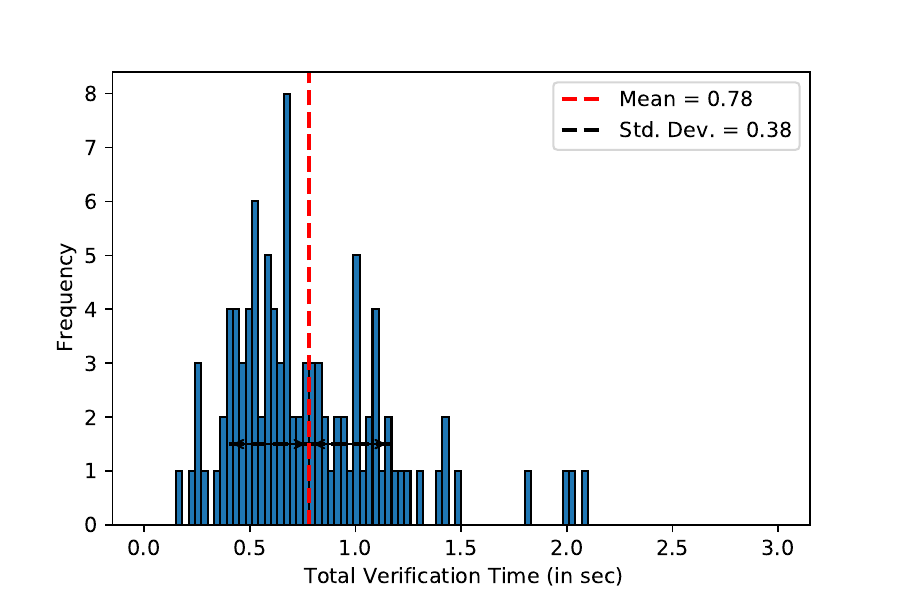}
\caption{\textit{German (2,4)}}
    \label{fig:german:verify}\end{subfigure}\\

\begin{subfigure}[b]{0.345\linewidth}
    \centering
     \includegraphics[width=\linewidth]{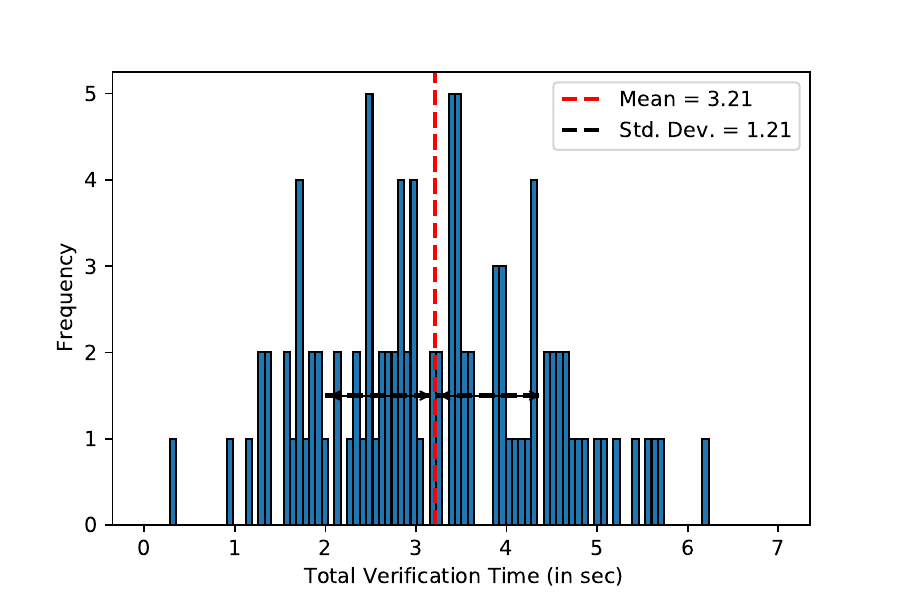}
        \caption{\textit{Credit (4,2)}}
    \label{fig:credit:verify}
\end{subfigure}
\begin{subfigure}[b]{0.305\linewidth}
\centering \includegraphics[width=\linewidth]{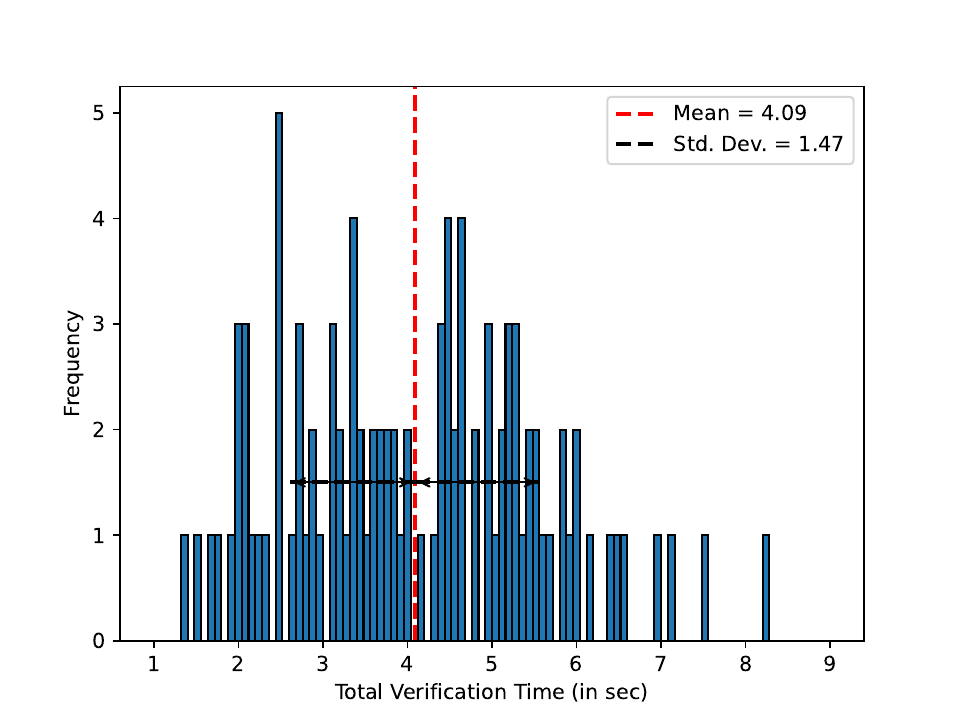}   
\caption{\textit{Adult (4,2)}}
    \label{fig:adult:verify}\end{subfigure}
\begin{subfigure}[b]{0.345\linewidth}
\centering    \includegraphics[width=\linewidth]{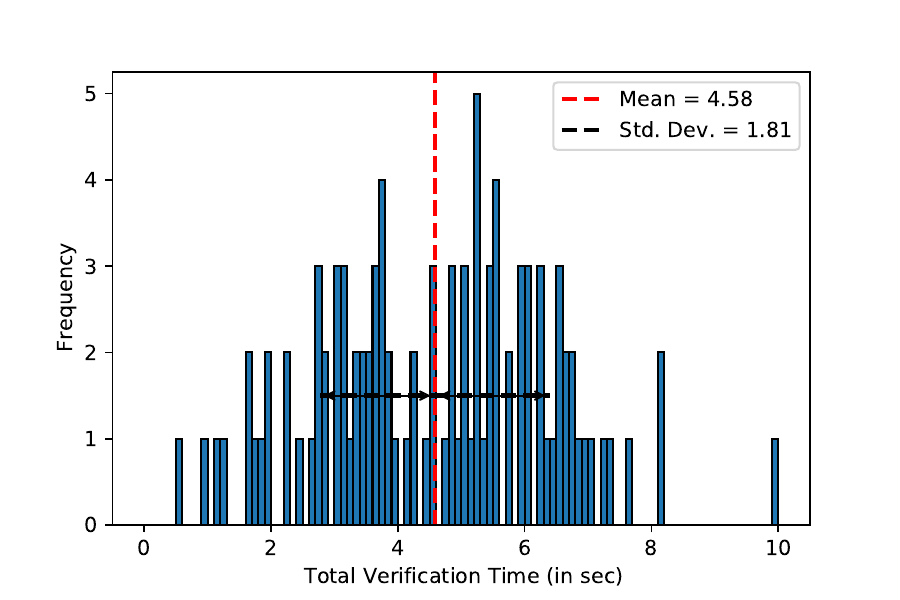}
\caption{\textit{German (4,2)}}
    \label{fig:german:verify}\end{subfigure}\\
    
     \caption{Distribution of verification time for 100 random data points.} 
     \label{fig:verification}
    \end{figure*}

\begin{figure*}[hbt!] 
\center
\begin{subfigure}[b]{0.32\linewidth}
    \centering
     \includegraphics[width=\linewidth]{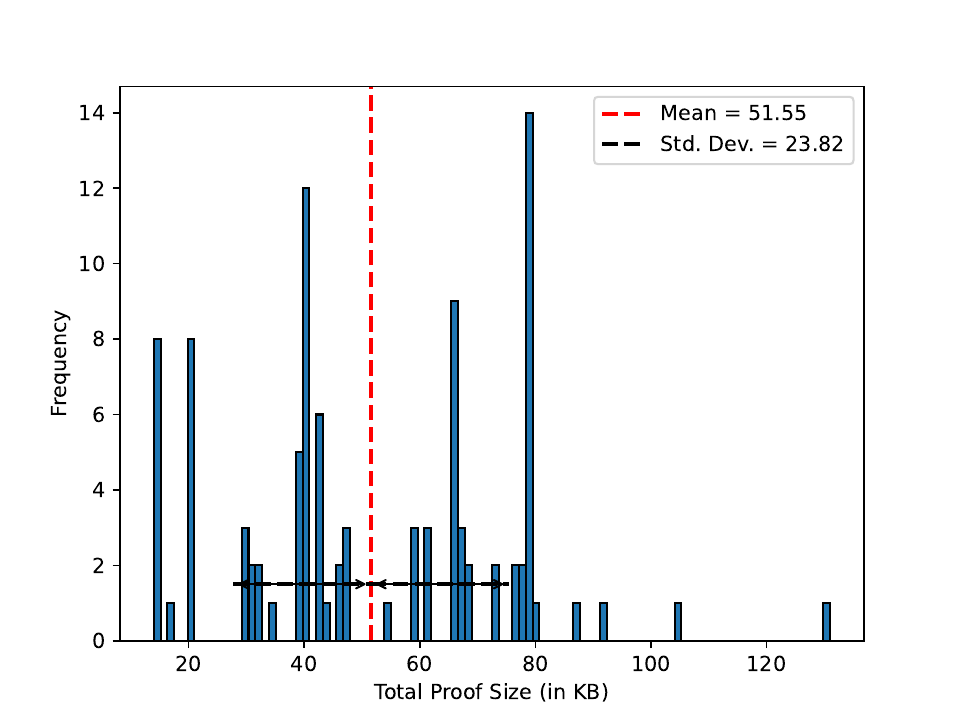}
        \caption{\textit{Credit (2,4)}}
    \label{fig:credit:communication}
\end{subfigure}
\begin{subfigure}[b]{0.32\linewidth}
\centering \includegraphics[width=\linewidth]{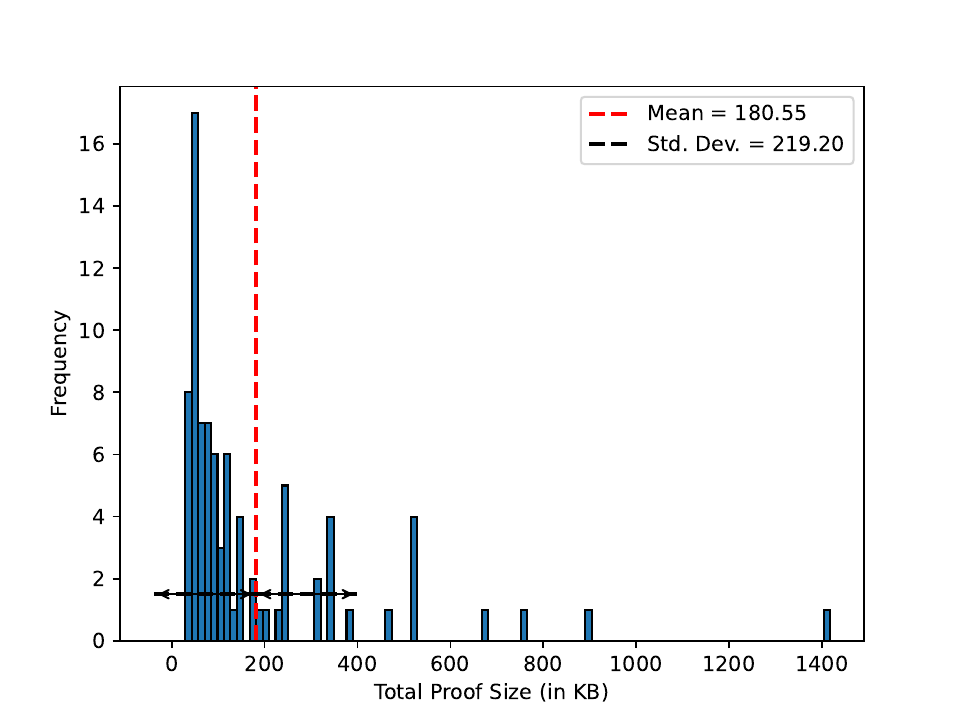}   
\caption{\textit{Adult (8,2)}}
    \label{fig:adult:communicaion}\end{subfigure}
\begin{subfigure}[b]{0.36\linewidth}
\centering    \includegraphics[width=\linewidth]{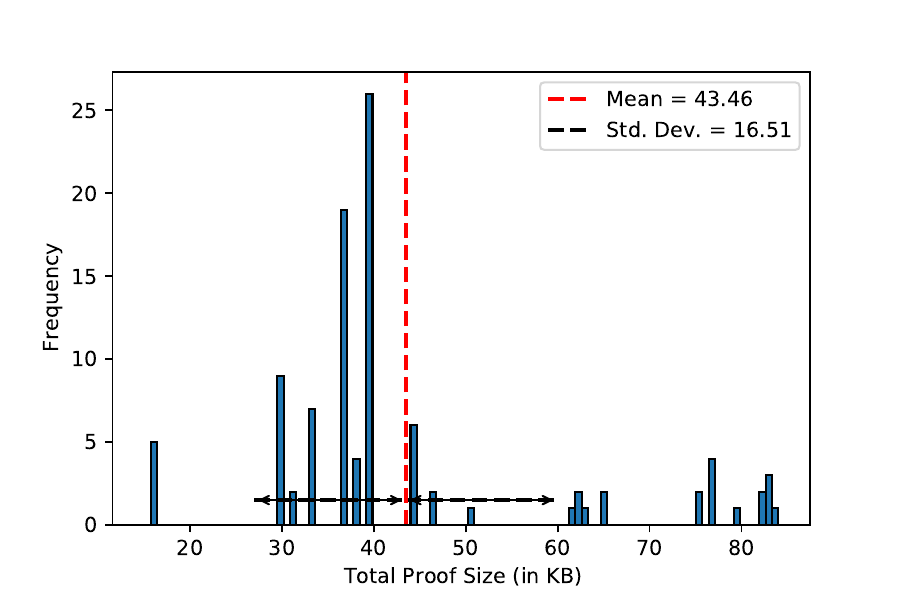}
\caption{\textit{German (2,4)}}
    \label{fig:german:communication}\end{subfigure}\\

\begin{subfigure}[b]{0.345\linewidth}
    \centering
     \includegraphics[width=\linewidth]{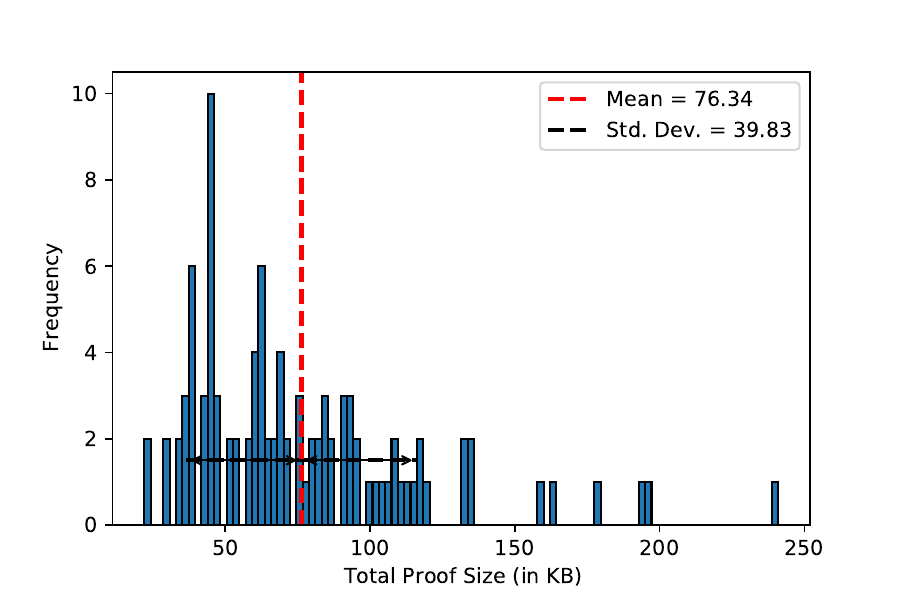}
        \caption{\textit{Credit (4,2)}}
    \label{fig:credit:communication}
\end{subfigure}
\begin{subfigure}[b]{0.305\linewidth}
\centering \includegraphics[width=\linewidth]{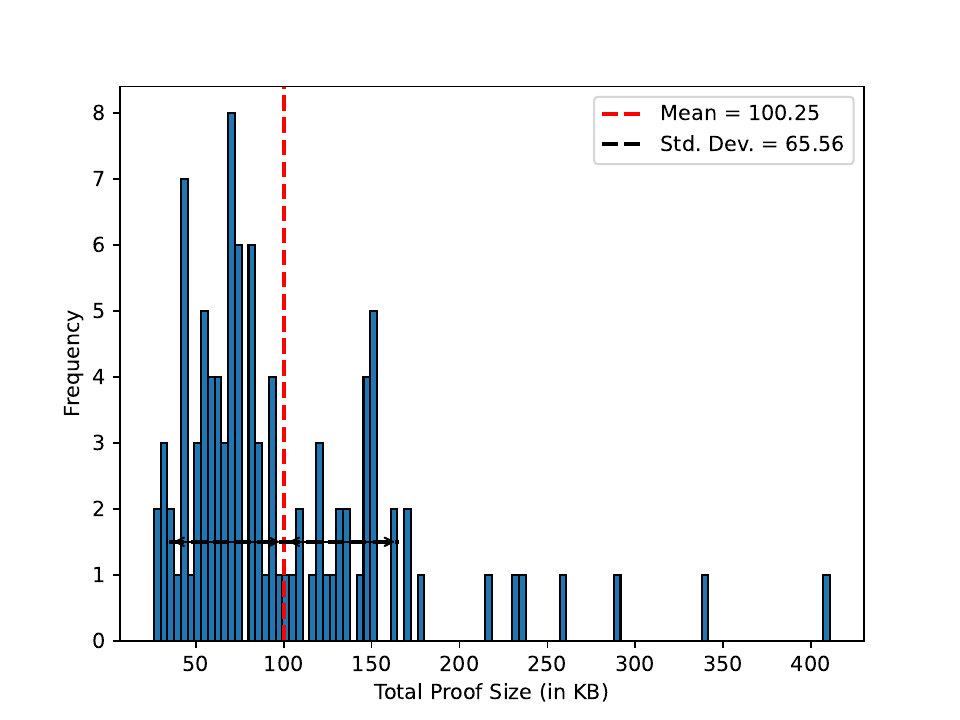}   
\caption{\textit{Adult (4,2)}}
    \label{fig:adult:communicaion}\end{subfigure}
\begin{subfigure}[b]{0.345\linewidth}
\centering    \includegraphics[width=\linewidth]{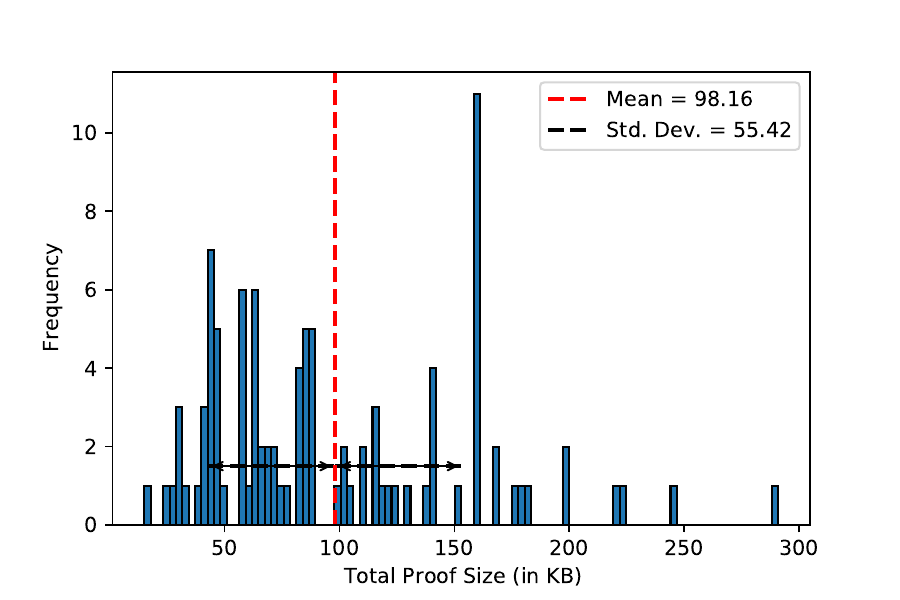}
\caption{\textit{German (4,2)}}
    \label{fig:german:communication}\end{subfigure}\\
    
     \caption{Distribution of communication cost (proof size) for 100 random data points.} 
     \label{fig:communication}
    \end{figure*}
    
\clearpage

\section{Related Work}\label{sec:related work appendix}

\noindent \textbf{Certifiable fairness.} Prior research on certifying fairness of a ML model can be classified into three types. The first line of work issues a certificate of fairness directly from the model weights by framing it as an optimization problem.  \cite{john2020verifying} presented optimization based mechanisms for certifying the (global) individual fairness of 
linear classifiers and kernelized classifiers with polynomial/rbf kernels. \cite{Benussi2022IndividualFG} extended the results to neural networks by encoding (global) individual fairness certification as  a mixed-integer linear programming problem. \cite{kang2022certifying} proposed a notion of distributional fairness and give a framework to compute provable certificates for the same. 

The second line of research has leveraged the connection between robustness and fairness, and proposed fairness-aware training mechanisms akin to adversarial training.  \cite{Ruoss2020} deviced a mechanism for training individually fair representations which can be used to
obtain a certificate of individual fairness for the end-to-end model by proving local robustness.  SenSR ~\cite{Yurochkin2020Training} introduced a distributionally robust optimization approach to enforce individual fairness on a model during training. CertiFair~\cite{khedr2022certifair} enabled certification of (global) individual fairness using off-the-shelf neural network verifiers. Additionally, the authors proposed a fairness aware training methodology with a  modified reguralizer. \cite{smoothing} applied randomized smoothing from adversarial robustness to make neural networks individually fair under a given weighted $\ell_p$ metric. \cite{doherty2023individual} estimated the (global) individual fairness parameter for Bayesian neural networks by designing Fair-FGSM and Fair-PGD -- fairness-aware extensions to gradient-based adversarial attacks for BNNs. 

The final line of work is based on learning theoretic approaches~\cite{yadav2022learningtheoretic,yan2022active,AVOIR} where a third-party audits the fairness of a model in a query-efficient manner. 

The problem of fairness certification has also garnered attention from the formal verification community. 
FairSquare~\cite{FairSquare} encoded a range of global fairness definitions as probabilistic program properties and provides a tool for automatically certifying that a program meets a given fairness property. VeriFair~\cite{Bastani2019} used adaptive concentration inequalities to design a probabilistically sound global fairness certification mechanism for neural networks. \cite{Urban2020} proposes a static analysis framework for certifying fairness of feed-forward neural networks.  Justicia~\cite{Ghosh2020JusticiaAS} presents a stochastic satisfiability
framework for formally verifying different group fairness measures, such as disparate impact, statistical parity, and equalized odds, of supervised
learning algorithms. A recent work, Fairify~\cite{Fairify}, generates a certificate for the global individual fairness of a pre-trained neural network using SMT solvers. 
It is important to note that all the aforementioned approaches focus on certification in the plain text, i.e., they do not preserve model confidentiality.

\noindent\textbf{Verifiable machine learning.} A growing line of work has been using cryptographic primitives to verify certain properties of a ML model without violating its confidentiality. Prior research has primarily focused on verifying the inference and accuracy of models. For instance, 
\cite{ZKDT} proposed a zero-knowledge protocol for tailored for verifying decision trees, while zkCNN~\cite{Liu2021zkCNNZK} introduced an interactive
protocol for verifying model inference for convolutional neural networks. Several other works have focused on non-interactive zero-knowledge inference for neural networks, including \cite{PvCNN, VI2, kang2022scaling, sun2023zkdl, Zen, vCNN}.   VeriML~\cite{VeriML} enabled the verification of the training process of a model that has been outsourced to an untrusted third party. \cite{PoT} proposed a mechanism for generating a cryptographic proof-of-training for logistic regression. 

Most of the prior work on verifying fairness while maintaining model confidentiality~\cite{pentyala2022privfair,BlindJustice,Toreini2023VerifiableFP,Segal21,Park22} has approached the problem in the third-party auditor setting. A recent work \cite{confidant} proposed a fairness-aware training pipeline for decision trees that allows the model owner to cryptographically prove that the learning
algorithm used to train the model was fair by design. In contrast, \name~allows a model owner to issue a certificate of fairness of neural networks by simply inspecting the model weights post-training.

\end{document}